\documentclass{article}

\PassOptionsToPackage{numbers,compress}{natbib}
\usepackage[preprint]{neurips_2019}

\usepackage[T1]{fontenc}
\usepackage{hyperref}
\usepackage{url}
\usepackage{booktabs}
\usepackage{amsfonts}
\usepackage{nicefrac}
\usepackage{microtype}

\usepackage{graphicx}
\usepackage{caption}
\usepackage{subcaption}

\usepackage{algorithm}
\usepackage{algorithmic}
\usepackage{amssymb}
\usepackage{amsthm}
\usepackage{bm}
\usepackage{enumitem}
\usepackage{mathtools}
\usepackage{xcolor}

\newtheorem{theorem}{Theorem}
\newtheorem{lemma}{Lemma}
\newtheorem{corollary}{Corollary}
\newtheorem{fact}{Fact}
\newtheorem{remark}{Remark}

\allowdisplaybreaks

\ifodd 1 \newcommand{\com}[1]{{\color{red}(C: #1)}}
\else \newcommand{\com}[1]{} \fi

\ifodd 0 \newcommand{\rev}[1]{{\color{blue}#1}}
\else \newcommand{\rev}[1]{#1} \fi

\ifodd 0 \newcommand{\reva}[1]{{\color{blue}#1}}
\else \newcommand{\reva}[1]{#1} \fi

\ifodd 0 
\else  \fi

\ifodd 1 \newcommand{\dtl}[1]{{\noindent\color{red}Details:\\#1}}
\else \newcommand{\dtl}[1]{} \fi

\ifodd 0 \newcommand{\red}[1]{{\color{red}#1}}
\else \newcommand{\red}[1]{#1} \fi

\DeclareMathOperator{\argmax}{argmax}

\newcommand{\bs}[1]{\bm{#1}}

\renewcommand{\emptyset}{\text{\O}}
\renewcommand{\epsilon}{\varepsilon}
\newcommand{\Ex}{\mathbb{E}}
\newcommand{\In}{\mathbb{I}}
\newcommand{\lex}{\mathrm{lex}}
\renewcommand{\Pr}{\mathbb{P}}
\newcommand{\Reg}{\mathrm{Reg}}

\renewcommand{\citet}{\cite}

\title{Lexicographic Multiarmed Bandit}

\author{%
    Alihan H\"uy\"uk \\
    Department of Electrical and \\ Electronics Engineering \\
    Bilkent University \\
    Ankara, Turkey \\
    \texttt{alihan.huyuk@ug.bilkent.edu.tr} \\
    \And
    Cem Tekin \\
    Department of Electrical and \\ Electronics Engineering \\
    Bilkent University \\
    Ankara, Turkey \\
    \texttt{cemtekin@ee.bilkent.edu.tr} \\
}

\begin{document}

\maketitle

\begin{abstract}
We consider a multiobjective multiarmed bandit problem with lexicographically ordered objectives. In this problem, the goal of the learner is to select arms that are lexicographic optimal as much as possible without knowing the arm reward distributions beforehand. We capture this goal by defining a multidimensional form of regret that measures the loss of the learner due to not selecting lexicographic optimal arms, and then, consider two settings where the learner has prior information on the expected arm rewards. 
In the first setting, the learner only knows for each objective the lexicographic optimal expected reward. In the second setting, it only knows for each objective near-lexicographic optimal expected rewards.
For both settings we prove that the learner achieves expected regret uniformly bounded in time. \reva{The algorithm we propose for the second setting also attains bounded regret for the multiarmed bandit with satisficing objectives.} In addition, we also consider the harder prior-free case, and show that the learner can still achieve sublinear in time gap-free regret. Finally, we experimentally evaluate performance of the proposed algorithms in a variety of multiobjective learning problems.
\end{abstract}

\section{Introduction} \label{sec:intro}

A vast number of decision making and learning tasks involve multidimensional performance metrics (objectives). Examples include recommending items in a recommender system to optimize accuracy, diversity and novelty \cite{zhou2010solving,konstan2006lessons}, learning lexicographic optimal routing flows in wireless networks \cite{shah2009lexicographically}, and adjusting the dose of radiation therapy for cancer patients while prioritizing target coverage over proximity of the therapy to the organs at risk \cite{lin2007intensity}. In the aforementioned problems, the learner aims to choose arms that yield high rewards in all of the objectives; however, it prefers arms that yield high rewards in the low priority objectives only if they do not compromise the rewards in the high priority objectives. 

Motivated by the above applications, in this paper we propose a new multiarmed bandit (MAB) problem called the lexicographic multiarmed bandit (Lex-MAB). In this problem, the learner's priority over the objectives is formally captured by lexicographic ordering. Essentially, given $D$ objectives indexed by the set ${\cal D} := [D]$, objective $i$ has a higher priority than objective $j$ if $i<j$.\footnote{For a positive integer $D$, $[D] := \{1,\ldots,D\}$} This priority induces a preference over the finite set of arms denoted by ${\cal A}$. Namely, given two arms $a$ and $a'$ with the corresponding real-valued expected reward vectors $\bm{\mu}_a:=(\mu^1_a,\ldots,\mu^D_a)$ and $\bm{\mu}_{a'}:=(\mu^1_{a'},\ldots,\mu^D_{a'})$, we say that arm $a$ lexicographically dominates arm $a'$ in the first $i \leq D$ objectives (written as $a \succ_{\lex,i} a'$) if $\mu^j_a>\mu^j_{a'}$, where $j:=\min\{k\leq i: \mu^k_{a} \neq \mu^k_{a'} \}$.\footnote{If there is no such $j$, then $\mu^k_a=\mu^k_{a'}$ for all $k\in [i]$, which implies that $\bm{\mu}_a$ does not lexicographically dominate $\bm{\mu}_{a'}$ in the first $i$ objectives.} Here, the latter expression is succinctly expressed as $\bm{\mu}_a \succ_{\lex,i}\bm{\mu}_{a'}$. Based on this preference, the set of lexicographic optimal arms are defined as the ones that are not lexicographically dominated by any other arm in all of the $D$ objectives, which is given as 
${\cal A}_* :=  \{a\in\mathcal{A}: \bm{\mu}_{a'} \nsucc_{\lex,D} \bm{\mu}_a, \forall a'\in\mathcal{A}\}$.

In the Lex-MAB, in each round $t$ the learner selects an arm $a(t)=a$ from ${\cal A}$, and then, receives a $D$-dimensional random reward vector $\bs{r}(t) = (r^1(t),\ldots,r^D(t))$ that is drawn from a fixed distribution with expectation vector $\bs{\mu}_a$. The goal of the learner is to perform as well as an oracle which perfectly knows the set of lexicographic optimal arms and selects a lexicographic optimal arm in every round. We capture the ordering of the objectives by introducing a multidimensional regret measure called the lexicographic regret. As this regret notion is fundamentally different from the scalar regret notion used in the classical stochastic MAB \cite{lai1985asymptotically}, minimizing it requires exploiting the multidimensional nature of the rewards and ordering of the objectives both in algorithm design and technical analysis. 

This is a challenging task because simple techniques such as turning the problem into an MAB with scalar rewards by using scalarizing methods from multiobjective optimization \cite{ehrgott2005multicriteria} will not work since the solution of the scalarized problem may not produce lexicographic optimal arms. The problem is further complicated due to the fact that without any prior knowledge on the expected arm rewards, it is impossible to identify lexicographic optimal arms with high probability. This can be observed by considering a problem instance with arms $a$ and $b$, and $D=2$, such that $a$ and $b$ have the same expected reward in objective $1$ and $a$ is the only lexicographic optimal arm. Although in this problem the learner can identify with high probability which arm is better in the second objective, it can never be sure about the lexicographic optimality of this arm. We call this problem the {\em identifiability problem}.\footnote{It also exists in the best arm identification problem with a fixed confidence (see Chapter 33 in \citet{LS19bandit-book}).}

The challenges described above motivates us to focus on the cases when the learner has prior knowledge on expected rewards. Specifically, we consider two types of prior knowledge, which generalize the prior knowledge introduced in \citet{bubeck2013bounded} and \citet{vakili2013achieving} to multidimensional rewards. In the first case, we assume that the expected rewards of a lexicographic optimal arm are known. In the second case, we assume that near-lexicographic optimal expected rewards are known. Then, we build learning algorithms that utilize the prior information to achieve uniformly bounded in time lexicographic regret for both cases. 

Importantly, for the first case, we show that the regret in each objective due to selecting a suboptimal arm $a$ is inversely proportional to the maximum of the gaps of arm $a$ over all objectives. This shows that having prior information over multiple objectives speeds up elimination of suboptimal arms. This is analogous to the combinatorial MAB \cite{gai2012combinatorial} in the sense that observations from one objective can help ruling out suboptimal arms in other objectives. We also prove that a similar gain appears in the second case, albeit we cannot rule out an arm performing much better than a lexicographic optimal arm in some objective as suboptimal. 

In addition to the cases with prior information, as the third case, we consider learning without any prior information. In this setting, despite the identifiability problem, we show that it is possible to achieve $\tilde{O}(T^{2/3})$ gap-free regret by learning to select near-lexicographic optimal arms. Finally, we numerically evaluate the performance of the learner with and without prior information on several multiobjective learning problems with lexicographically ordered objectives. 


Rest of the paper is organized as follows. Related work is given in Section \ref{sec:related}. The Lex-MAB, the lexicographic regret and two types of prior information are defined in Section \ref{sec:problem}. \rev{Algorithms and regret bounds for the Lex-MAB under prior information and for the prior-free case are considered in Section \ref{sec:case1}.}
Experimental results are given in Section \ref{sec:experiments} followed by the concluding remarks in Section \ref{sec:conc}. \rev{Proofs, additional simulations, tables of notation and the code for experimental results are given in the appendix.}

\section{Related Work} \label{sec:related}

\citet{lai1985asymptotically} shows that in the classical stochastic MAB problem, for any uniformly good policy the regret grows at least logarithmically over time. As opposed to this, \citet{lai1984optimal} proves for the two-armed stochastic bandit that when the learner has prior information on the maximum expected reward $\mu^*$ and the minimum nonzero suboptimality gap $\Delta$, there exists policies that can achieve uniformly bounded regret. This idea is further investigated in \citet{bubeck2013bounded}, which shows that bounded regret of order $1/\Delta$ is achieved for the case with finitely many arms when the learner knows $\mu^*$ and a positive lower bound on $\Delta$. 
\rev{\citet{garivier} studies the case where only $\mu^*$ is known and proposes an algorithm with bounded regret of order $\log(1/\Delta) (1/\Delta)$, and also proves a lower bound of order $1/\Delta$. This paper also provides a generic tool to prove gap-dependent lower bounds on the regret.} Similarly, \citet{bubeck2013prior} considers Thompson sampling and shows that its regret is uniformly bounded when $\mu^*$ and a positive lower bound on $\Delta$ are known. On the other hand, \citet{vakili2013achieving} considers a weaker prior information model where the learner knows a near-optimal expected reward $\eta$, which can be computed using $\mu^*$ and a positive lower bound on $\Delta$. The proposed algorithm obtains $\sum_a \Delta_a/\delta^3$ regret, where $\delta=\mu^*-\eta<\Delta$ and $\Delta_a$ is the suboptimality gap of arm $a$.
\citet{mersereau2009structured} and \citet{lattimore2014bounded} consider as prior information the knowledge of parametrized expected reward functions for each arm. In these works, the only unknown is the true parameter, which can be estimated by using reward observations from all of the arms. 

Different from the works mentioned above, in this paper we consider a multiobjective MAB problem with lexicographically ordered objectives. We design algorithms that exploit the prior information in all objectives simultaneously to rule out arms that are not lexicographic optimal. \rev{Our regret bounds match the one in \cite{garivier} and improve the one in \cite{vakili2013achieving} for the case with a single objective.}

Numerous works have investigated regret minimization in multiobjective variants of the MAB problem. For instance, \citet{drugan2013designing} defines for each suboptimal arm its distance to the Pareto front as the Pareto suboptimality gap and the regret as the sum of the Pareto suboptimality gaps of the arms chosen by the learner. It proposes a learning algorithm that achieves $O(\log T)$ gap dependent Pareto regret. \citet{turgay2018multi} considers the multiobjective contextual MAB problem with similarity information, and extends the contextual zooming algorithm in \citet{slivkins2014contextual} to minimize the Pareto regret while making fair selections among the estimated Pareto optimal arms. The proposed algorithm is shown to achieve $\tilde{O}(T^{(1+d_p)/(2+d_p)})$ Pareto regret where $d_p$ is the Pareto zooming dimension. In addition, \citet{tekin2018multi} considers a biobjective contextual MAB problem with lexicographically ordered objectives. Unlike this work, we study the general case with $D$ lexicographically ordered objectives and focus on the importance of prior information in learning. 


Our problem is also related to the tresholding MAB studied in \citet{locatelli2016optimal} and satisficing MAB studied in \citet{reverdy2017satisficing}. Specifically, if the prior information is given with respect to a target arm which is not necessarily a lexicographic optimal arm, and the regret in each objective is defined with respect to the expected reward of the target arm in that objective, then our algorithms for the cases with prior information can easily be adapted to minimize the regret with respect to the target arm. \reva{Specifically, we adapt our second case to the satisfaction-in-mean-reward problem introduced in \cite{reverdy2017satisficing} and prove that our algorithm achieves bounded regret for this problem as well.}\footnote{\reva{See Section \ref{sec:satisficing} for our results on the multiarmed bandit with satisficing objectives.}}
 

\section{Problem Formulation} \label{sec:problem}


\textbf{System Model:}
We consider rounds indexed by $t\in\{1,2,\ldots\}$. In each round $t$, the learner first selects an arm $a(t)$ from the finite arm set $\mathcal{A} := [A]$, and then, observes a random reward for each objective $i \in \mathcal{D}:= [D]$, denoted by $r^i(t)$, which is equal to $\mu_{a(t)}^i+\kappa^i(t)$, where $\mu_a^i$ denotes the expected reward of arm $a$ in objective $i$ and $\kappa^i(t)$ denotes the zero mean noise. The learner does not know the expected reward vector $\bm{\mu}_a:=(\mu_a^1,\ldots,\mu_a^D)$, $a\in\mathcal{A}$ beforehand, and given $a(t)=a$, the noise vector $\{\kappa^1(t),\ldots,\kappa^D(t)\}$ is sampled from a fixed (unknown) multivariate distribution $\bs{\nu}_a$, independent of the other rounds. Moreover, its marginal distribution is 1-sub-Gaussian, i.e., $\forall a \in {\cal A}$ and $\forall\lambda\in\mathbb{R}, \Ex[e^{\lambda\kappa^i(t)}|a(t)=a] \leq \exp(\lambda^2/2)$.\footnote{Noise can be dependent over the objectives.} The assumption on the noise distribution is very general as it covers the Gaussian distribution with zero mean and unit variance, and any bounded zero mean distribution defined over an interval of length 2.

\textbf{Lexicographic Optimality:}
For two $D$-dimensional real-valued vectors $\bm{\mu}:=(\mu^1,\ldots,\mu^D)$ and $\bm{\mu'}:=(\mu'^1,\ldots,\mu'^D)$, and $i \in [D]$, we say that $\bm{\mu}$ lexicographically dominates $\bm{\mu'}$ in the first $i$ objectives, denoted by $\bm{\mu}\succ_{\lex,i}\bm{\mu'}$, if $\mu^j>\mu'^j$, where $j:=\min\{k\leq i: \mu^k\neq\mu'^k\}$. Based on this, we say that arm $a$ lexicographically dominates arm $a'$ in the first $i$ objectives if $\bm{\mu}_a\succ_{\lex,i}\bm{\mu}_{a'}$. \rev{The complement of this is denoted by $\bm{\mu}_a\nsucc_{\lex,i}\bm{\mu}_{a'}$.}

Let $\mathcal{A}_*^i := \{a\in\mathcal{A}: \bm{\mu}_{a'}\nsucc_{\lex,i}\bm{\mu}_a, \forall a'\in\mathcal{A}\}$ denote the set of lexicographic optimal arms in the first $i$ objectives and define ${\cal A}_* := \mathcal{A}_*^D$. Clearly, we have $\mathcal{A}_*^{i+1}\subseteq\mathcal{A}_*^i$ for $i\in [D-1]$. We use $*$ to denote an arm that is lexicographic optimal in all objectives, and $\mu_*^i$ to denote the expected reward of this arm in objective $i$. Moreover, we define the \textit{\rev{gap}} of arm $a$ in objective $i$ as \rev{$\Delta_a^i:=\mu_*^i-\mu_a^i$} and the \textit{\rev{absolute gap}} of arm $a$ in objective $i$ as $\nabla_a^i := |\mu_*^i-\mu_a^i|$.
For $i\in\{2,\ldots,D\}$, we let $\mathcal{S}_*^i := \mathcal{A}_*^{i-1}-\mathcal{A}_*^i$ denote the set of arms that are lexicographic optimal in the first $i-1$ objectives but not lexicographic optimal in the first $i$ objectives and define $\mathcal{S}_*^1:=\mathcal{A}-\mathcal{A}_*^1$. \rev{Note that $a\in\mathcal{S}_*^i$ implies that $\Delta_a^i>0$.} \rev{The set of suboptimal arms in objective $i$ is given as $\mathcal{S}^i:=\{a:\Delta_a^i>0\}$.}
We also define the maximum suboptimality gap in objective $i$ as $\Delta_{\max}^i := \max_{a \in \mathcal{A}} \Delta_a^i$ and \rev{the maximum absolute gap as $\nabla_a^{\max}:=\max_{j\in\mathcal{D}}\nabla_a^j$.} 
Finally, let $\Delta_{\min}^i := \min_{a \in \mathcal{S}_*^i} \Delta_a^i$ denote the minimum suboptimality gap among arms in $\mathcal{S}_*^i$.\footnote{If $S_*^i=\emptyset$, then $\Delta_{\min}^i=\infty$.}

\textbf{Types of Prior Knowledge:}

\textbf{(Case 1)} Lexicographic optimal expected rewards are known: The learner knows $\mu_*^i$ for all $i\in\mathcal{D}$. For this case, we assume $\mu_*^i =0$, $\forall i\in\mathcal{D}$ without any loss of generality.

\textbf{(Case 2)} Near-lexicographic optimal expected rewards are known: The learner knows $\eta_i$ such that $\mu_*^i - \Delta_{\min}^i < \eta_i < \mu_*^i $ for all $i\in\mathcal{D}$. For this case, we define $\delta_i:=\mu_*^i-\eta_i, \forall i\in\mathcal{D}$ and assume $\eta_i=0, \forall i\in\mathcal{D}$ without any loss of generality.

\textbf{(Case 3)} No prior knowledge on the expected rewards. 


\red{
\begin{remark}\label{rmk:normalization}
In Case 1, if $\mu_*^i$s are not equal to $0$, we can subtract them from the rewards to obtain $\tilde{r}^i(t) := r^i(t)-\mu_*^i$. Under these transformed rewards, we will have $\tilde{\mu}_*^i=0, \forall i\in\mathcal{D}$ and the gaps that we have defined will not be affected. Similarly, in Case 2, we can subtract $\eta_i$s from the rewards to obtain $\tilde{r}^i(t):=r^i(t)-\eta_i$. 
\end{remark}
}

\textbf{Regret Definitions:}
The (pseudo) regret of the learner is measured with respect to an oracle, which knows the expected rewards of the arms, and chooses a lexicographic optimal arm in each round. We define two notions of regret: Priority-based and priority-free regrets in objective $i$ are given as $\Reg^i_{pb}(T) := \sum_{t=1}^T \Delta_{a(t)}^i\In\{a(t)\in\mathcal{S}_*^i\}$ and $\Reg^i_{pf}(T) := \sum_{t=1}^T \Delta_{a(t)}^i$ respectively. 
The lexicographic priority-based and priority-free regrets are defined as the tuples $\bm{\Reg}_{pb}(T):=(\Reg^1_{pb}(T),\ldots,\Reg^D_{pb}(T))$ and $\bm{\Reg}_{pf}(T):=(\Reg^1_{pf}(T),\ldots,\Reg^D_{pf}(T))$ respectively. Subscripts will be removed from the notation when the considered regret notion is clear from the context. 

For $\bm{\Reg}_{pb}(T)$, when $a(t)\in\mathcal{S}_*^i$, regret is incurred only in objective $i$. No regret is incurred for $j<i$ since $\Delta_{a(t)}^j=0$. In addition, no regret is incurred for $j>i$ when \red{$a(t) \in \mathcal{S}_*^i$}. This definition of regret is consistent with the priority that the learner assigns to each objective. Since lexicographic ordering implies that even a small improvement in the expected reward in objective $i$ is more important than any improvement in the expected rewards of objectives $j>i$, the learner does not care about the loss it incurs in higher indexed objectives when $a(t)\in\mathcal{S}_*^i$. 
For $\bm{\Reg}_{pf}(T)$, an arm $a$ for which $\mu^i_a > \mu^i_*$ will incur negative regret in objective $i$, but positive regret in some other objective. 

We say that the regret is $O(\max\{f_1(T),\ldots,f_D(T)\})$ when $\Reg^i(T)\in O(f_i(T))$ for $i\in\mathcal{D}$. Under both notions of regret, the (cumulative) regret of any arm selection strategy cannot lexicographically dominate the cumulative regret of always selecting a lexicographic optimal arm, which is essentially the zero vector. Therefore, the time-averaged expected rewards of any algorithm that achieves sublinear $\bm{\Reg}_{pb}(T)$ or $\bm{\Reg}_{pf}(T)$ will converge (as $T \rightarrow \infty$) to the lexicographic optimal expected rewards. In addition, under $\bm{\Reg}_{pf}(T)$ the lexicographic ordering between the cumulative expected rewards and the regrets of any pair of sequences of arms $(a(1),\ldots,a(T))$ and $(a'(1),\ldots,a'(T))$ will be the same.

\begin{remark}
	For Case 1, the bound for $\bm{\Reg}_{pb}(T)$ given in Theorem \ref{thm:OM-LEX} also holds for $\bm{\Reg}_{pf}(T)$ if we replace $\mathcal{S}_*^i$ with $\mathcal{S}^i$. For Case 2, if we redefine the minimum suboptimality gap as 
	$\Delta_{\min}^i := \min_{a \in \mathcal{S}^i} \Delta_a^i$ (which implies a stronger prior knowledge on the near-lexicographic optimal expected rewards), and replace $\mathcal{S}_*^i$ with $\mathcal{S}^i$, then the bound for $\bm{\Reg}_{pb}(T)$ in Theorem \ref{thm:NOM-LEX} also holds for $\bm{\Reg}_{pf}(T)$. Theorems \ref{thm:OM-LEX-supp} and \ref{thm:NOM-LEX-supp} in the appendix bound $\bm{\Reg}_{pf}(T)$ for Cases 1 and 2 respectively. On the other hand, for Case 3, the regret analysis only holds for $\bm{\Reg}_{pf}(T)$ (see Section \ref{sec:case1} for the details).
\end{remark}

\section{Learning Algorithms and Regret Bounds}\label{sec:case1}

\textbf{A Learning Algorithm for Case 1:} 
We propose {\em Optimal Mean based Lexicographic Exploration and eXploitation} (OM-LEX) given in Algorithm \ref{algorithm:OM-LEX} for the prior information described in Case 1. In essence, OM-LEX generalizes the arm selection rule proposed in Algorithm 1 in \citet{garivier} to multiple objectives. It keeps for each arm $a$ a counter $N_a$ that counts how many times arm $a$ was selected prior to the current round and the sample mean estimate $\hat{\mu}_a^i$ of the rewards from objective $i$ of arm $a$ observed prior to the current round for all $i \in {\cal D}$. The values of these variables at the beginning of round $t$ are denoted by $N_a(t)$ and $\hat{\mu}_a^i(t)$ respectively. 

\begin{algorithm}
    \caption{OM-LEX}
    \label{algorithm:OM-LEX}
    \begin{algorithmic}[1]
        \STATE \textbf{Inputs:} $\mu_*^i = 0, \forall i\in\mathcal{D}$
        \STATE \textbf{Counters:} $N_a, \forall a\in\mathcal{A}$
        \STATE \textbf{Estimates:} $\hat{\mu}_a^i, \forall a\in\mathcal{A}, \forall i\in\mathcal{D}$
        \STATE For each round $t\in{1,\ldots,A}$, select arm $t$
        \STATE For each round $t>A$:
        \STATE \hspace{\algorithmicindent}Let $\hat{{\cal A}}_* := \{ a \in\mathcal{A}: \forall i\in\mathcal{D}, |\hat{\mu}_a^i| < \sqrt{4\log N_a/N_a} \}$
        \STATE \hspace{\algorithmicindent}If $\hat{{\cal A}}_* \neq \emptyset$, select an arm $a(t)$ in $\hat{{\cal A}}_*$ uniformly at random, update $t\gets t+1$
        \STATE \hspace{\algorithmicindent}If $\hat{{\cal A}}_* = \emptyset$, select $a(t)=1, a(t+1)=2, \ldots, a(t+A-1)=A$, update $t\gets t+A$
    \end{algorithmic}
\end{algorithm}

OM-LEX starts by selecting each arm exactly once. In the remaining rounds, it checks whether there exists an arm whose sample mean reward in objective $i$ is within a shrinking neighborhood of the lexicographic optimal arm's expected reward for all objectives $i\in\mathcal{D}$. For this, it computes the set of estimated lexicographic optimal arms in round $t$ as
\begin{align*}
    \hat{{\cal A}}_*(t) := \left\{ a \in\mathcal{A}: \forall i\in\mathcal{D}, |\hat{\mu}_a^i(t)| < \sqrt{\frac{4\log N_a(t)}{N_a(t)}} \right\} ~.
\end{align*}
If $\hat{{\cal A}}_*(t) \neq \emptyset$, then OM-LEX exploits by selecting one of the arms in $\hat{{\cal A}}_*(t)$ uniformly at random as it expects only the lexicographic optimal arms to satisfy this condition in the long run. If no such arm exists, then OM-LEX explores by playing all arms in a round-robin fashion.
The following theorem shows that the \rev{expected priority-based} regret of OM-LEX is uniformly bounded in time. 

\begin{theorem} \label{thm:OM-LEX}
	When OM-LEX is run,  $\forall i\in\mathcal{D}$ and $\forall T\geq 1$, we have
    \begin{align*}
        \Ex[\Reg^i_{pb}(T)] \leq \sum_{a\in\mathcal{S}_*^i} \left( \left(\frac{\pi^2}{3}D + 1 \right)\Delta_a^i+ \frac{36}{\nabla_a^{\max}}\log\frac{17}{\nabla_a^{\max}} \right) .
    \end{align*}
\end{theorem}

When $D=1$, this result is identical to the regret bound in Theorem 9 in \citep{garivier} except for some constants. In the multiobjective case, we see that the regret induced by an arm in one objective depends on the maximum of the absolute gaps of the same arm over all objectives. As long as the arm has a large absolute gap in at least one objective, it is easy to identify it as a suboptimal arm.

\textbf{A Learning Algorithm for Case 2:} 
We propose {\em Near Optimal Mean based Lexicographic Exploration and eXploitation} (NOM-LEX). NOM-LEX has almost the same structure with OM-LEX. Its pseudocode is exactly the same as Algorithm \ref{algorithm:OM-LEX} except two differences: Firstly, its input prior knowledge \rev{(given in line 1 of Algorithm \ref{algorithm:OM-LEX})} is \rev{$\eta_i=0$, $\forall i \in {\cal D}$}. Secondly, NOM-LEX computes the set of estimated lexicographic optimal arms in round $t$ \rev{(given in line 6 of Algorithm \ref{algorithm:OM-LEX})} as 
\begin{align*}
    \hat{{\cal A}}_*(t) := \left\{ a \in\mathcal{A}: \forall i\in\mathcal{D}, \hat{\mu}_a^i(t) > -\sqrt{\frac{4\log N_a(t)}{N_a(t)}} \right\} ~.
\end{align*}

The next theorem bounds the \rev{expected priority-based} regret of NOM-LEX.

\begin{theorem} \label{thm:NOM-LEX}
	When NOM-LEX is run, $\forall i\in\mathcal{D}$ and $\forall T\geq 1$, we have
    \begin{align*}
        \Ex[\Reg^i_{pb}(T)] \leq \sum_{a\in\mathcal{S}_*^i} \left( \rev{\left(\frac{\pi^2}{6}D + 1 \right)} \Delta_a^i+ \frac{36\Delta_a^i}{\rev{(\max_{j\in\mathcal{D}}(\Delta_a^j-\delta_j))^2}}\log\frac{17}{\max_{j\in\mathcal{D}}(\Delta_a^j-\delta_j)} \right) ~.
    \end{align*}
\end{theorem}

From Theorem \ref{thm:NOM-LEX}, we see that the regret due to a suboptimal arm $a$ in objective $i$ depends on the maximum squared difference between the suboptimality gaps of that arm and near-lexicographic optimal expected rewards over all objectives. This also shows that the prior knowledge in other objectives may help the learner attain smaller regret in objective $i$. However, since the lexicographic optimal expected rewards are not known, unlike Case 1, we cannot rule out a suboptimal arm in objective $i$ by observing that it is much better than a lexicographic optimal arm in another objective.

\textbf{A Learning Algorithm for Case 3:} 
We propose {\em Prior Free Lexicographic Exploration and eXploitation} (PF-LEX) given in Algorithm \ref{algorithm:PF-LEX}, which learns to select near-lexicographic optimal arms without any prior information on the mean arm rewards. 
PF-LEX takes as input $\epsilon>0$, which is proportional to the suboptimality that the learner aims to tolerate in \red{the objectives} (this will be adjusted based on the time horizon $T$). Similar to OMG-LEX, it keeps for each arm $a$ the counter $N_{a}(t)$ and the sample mean reward $\hat{\mu}^i_{a}(t)$, $\forall i \in \mathcal{D}$.

Arm selection of PF-LEX in round $t$ depends on the confidence intervals in the first $D-1$ objectives. The {\em upper confidence bound} (UCB) and the {\em lower confidence bound} (LCB) of arm $a$ in objective $i$ are given as $u_{a}^i(t) := \hat{\mu}_{a}^i(t) + c_{a}(t)$ and $l_{a}^i(t) := \hat{\mu}_{a}^i(t) - c_{a}(t)$ respectively. Here
\begin{align*}
c_{a}(t) := \sqrt{  \frac{1+N_a(t)}{N^2_a(t)} \left( 1+ 2 \log \left( \frac{AD \sqrt{1 + N_a(t) } } {\delta} \right) \right)  }
\end{align*}
represents the uncertainty in arm $a$'s reward, and $\delta$ is called the {\em confidence term}, which is also given as input to PF-LEX. As expected, the uncertainty decreases as arm $a$ gets selected. It is easy to see that $\mu^i_a \in [l_{a}^i(t), u_{a}^i(t)]$ with high probability for all objectives and all rounds.
In each round, PF-LEX estimates the set of near-lexicographic optimal arms. For this, similar to \citet{joseph2016fairness}, we say that arms $a$ and $b$ are \textit{linked} in objective $i$ if $[l_a^i(t),u_a^i(t)]\cap[l_b^i(t),u_b^i(t)]\neq\emptyset$. When $a$ and $b$ are in the same component of the transitive closure of the linked relation in objective $i$, we say they are \textit{chained} in objective $i$ and write $a\:C_{i,t}\:b$. Starting from $\hat{\mathcal{A}}_*^0(t) = \mathcal{A}$, PF-LEX recursively computes the estimate $\hat{{\cal A}}^i_*(t)$ of ${\cal A}^i_*$ for $i \in [D-1]$. After it computes $\hat{{\cal A}}_*^{i-1}(t)$, it identifies the optimistic near-lexicographic optimal arm in objective $i$ as $\hat{a}^i_*(t) = \argmax_{a \in \hat{{\cal A}}_*^{i-1}(t)} u_{a}^i(t)$. Then, it sets $\hat{{\cal A}}^i_*(t) = \{ a \in \hat{{\cal A}}_*^{i-1}(t): a\:C_{i,t}\:\hat{a}_*^i(t) \} $.

Suppose we always select $\hat{a}_*^D(t)$, which happens to be in ${\cal S}^i_*$ for some round $t$. For this case, we show in Lemma \ref{lmm:gap-bound} in the appendix that the regret incurred in this round is bounded by the length of the chain formed by $\hat{\mathcal{A}}_*^{i}(t)$. In order to guarantee regret that is proportional to $\epsilon$, we want the length of the chains not be more than a constant factor of $\epsilon$. As it is not always possible to shrink the chains by always selecting $\hat{a}_*^D(t)$, to achieve our goal, we require all arms in $\hat{\mathcal{A}}_*^1(t)$ to have narrow confidence intervals. Thus, PF-LEX selects $a(t)=a\in\hat{\mathcal{A}}_*^1(t)$ if there is an arm $a$ with high uncertainty, i.e., $c_a(t)>\epsilon/2$. On the other hand, if $c_a(t)\leq\epsilon/2$ for all $a\in\hat{\mathcal{A}}_*^1(t)$, then PF-LEX simply selects $a(t)=\hat{a}^D_*(t)$. Algorithm \ref{algorithm:PF-LEX} shows a more efficient implementation of PF-LEX that does not compute $\hat{{\cal A}}_*^j(t)$ for $j>1$ when $a(t)$ is selected from $\hat{{\cal A}}_*^1(t)$. Finally, After PF-LEX selects arm $a(t)$, it observes the random reward vector $\bs{r}(t) = (r^1(t), \ldots ,\allowbreak r^D(t))$ of arm $a(t)$, and then, updates the sample mean estimates of the rewards in objectives $i\in\mathcal{D}$ and the counter of $a(t)$. 

\begin{algorithm}
	\caption{PF-LEX}
	\label{algorithm:PF-LEX}
	\begin{algorithmic}[1]
		\STATE \textbf{Input:} $\epsilon$, $\delta$
		\STATE \textbf{Counters:} $N_{a}, \forall a\in\mathcal{A}$
		\STATE \textbf{Estimates:} $\hat{\mu}_{a}^i, \forall a\in\mathcal{A}, \forall i\in\mathcal{D}$
		\STATE For each round $t$:
		\STATE \hspace{\algorithmicindent}Compute $u_{a}^i=\hat{\mu}_{a}^i+c_a$ and $l_{a}^i=\hat{\mu}_{a}^i-c_a$, $\forall a\in\mathcal{A}$, $\forall i\in\mathcal{D}$
		\STATE \hspace{\algorithmicindent}Set $\hat{a}_*^1=\argmax_{a\in\mathcal{A}} u_a^1$, compute $\hat{\mathcal{A}}_*^1=\{a\in\mathcal{A}: a\:C_1\:\hat{a}_*^1\}$
		\STATE \hspace{\algorithmicindent}If there exists an arm $a$ in $\hat{\mathcal{A}}_*^1$ such that $c_a>\epsilon/2$:
		\STATE \hspace{2\algorithmicindent}Select an arm $a(t)$ in $\hat{\mathcal{A}}_*^1$ such that $c_{a(t)}>\epsilon/2$ uniformly at random
		\STATE \hspace{\algorithmicindent}If all arms $a$ in $\hat{\mathcal{A}}_*^1$ satisfies $c_a\leq\epsilon/2$:
		\STATE \hspace{2\algorithmicindent}Set $\hat{a}_*^i=\argmax_{a\in\hat{\mathcal{A}}_*^{i-1}} u_{a}^i$, compute $\hat{\mathcal{A}}_*^i=\{a\in\hat{\mathcal{A}}_*^{i-1}: a\:C_i\:\hat{a}^i_*\}$, $\forall i\in\{2,\ldots,D-1\}$
		\STATE \hspace{2\algorithmicindent}Select $a(t)=\hat{a}_*^D=\argmax_{a\in\hat{\mathcal{A}}_*^{D-1}} u_a^\rev{D}$
	\end{algorithmic}
\end{algorithm}

The following theorem shows that PF-LEX achieves $\tilde{O}(T^{2/3})$ regret. 

\begin{theorem} \label{theorem:gapindependent}
	When PF-LEX is run with $\delta \in (0,1)$ and $\epsilon > 0$, with probability at least $1-\delta$, for all $i\in\mathcal{D}$ and for all $T\geq 1$, we have
	\begin{align*}
		\Reg_{pb}^i(T) \leq 4\sqrt{2}B_{T,\delta}\sqrt{|{\cal S}^i_*|T} + \left( 3+\frac{16}{\epsilon^2}\log\frac{2\sqrt{e}AD}{\epsilon\delta} \right)|{\cal S}^i_*|\Delta^i_{\max} + \epsilon(A-1)T
	\end{align*}
	where $B_{T,\delta} := \sqrt{1+2\log(AD\sqrt{T}/\delta)}$.
	Given a particular time horizon $T$, by setting $\epsilon = T^{-1/3}$, with probability at least $1-\delta$, we have
	\begin{align*}
		\Reg_{pb}^i(T) \leq 4\sqrt{2}B_{T,\delta}\sqrt{|{\cal S}^i_*|T} + \left( 3+16T^{2/3}\log\frac{2\sqrt{e}ADT^{1/3}}{\delta} \right)|{\cal S}^i_*|\Delta^i_{\max} + (A-1)T^{2/3} ~.
	\end{align*}
	Moreover, taking $\delta=1/T$, $\Ex[\Reg_{pb}^i(T)] = \tilde{O}(T^{2/3})$.
\end{theorem}

\begin{remark}
	Unlike the cases with prior information, an analogue of the regret bound in Theorem \ref{theorem:gapindependent} will not hold for the priority-free regret when $\mathcal{S}^i_*$ is replaced by ${\cal S}^i$. Any two arms that are both lexicographic optimal in the first $i-1$ objectives are linked in these objectives with high probability. If one happens to be the selected arm, we are confident that they are both in $\hat{\mathcal{A}}_*^{i-1}$. When the selected arm is in $\mathcal{S}_*^i$, we use this fact and compare it to a lexicographic optimal arm to conclude that the gap of the selected arm in objective $i$ is smaller than the regret that we aim to tolerate. However, we fail to make any deductions about the higher indexed objectives.
\end{remark}

\section{Experiments} \label{sec:experiments}
We demonstrate our results in three different settings with $A=3$ and $D=2$. All rewards are assumed to come from independent Bernoulli distributions in both objectives with expected reward vectors given in Table~\ref{tbl:settings}. In all settings, the only lexicographic optimal arm is the first arm and $\Delta_{\min}^1=\Delta_{\min}^2=0.10$. We focus on the priority-based regret in this section. Additional experiments that also consider the priority-free regret are given in the appendix.

In Setting 1, apart from the lexicographic optimal arm, there is another arm that is also optimal in objective 1, which requires the learner to consider rewards in objective 2. However, the third arm makes this tricky. It is not only suboptimal in objective 1 but also has very high reward in objective 2.
Setting 2 is specifically designed to be challenging for Cases 1 and 2. Since arms that are not lexicographic optimal are suboptimal in exactly one objective, eliminating arms based on information from the other objective is not possible.
Setting 3 \red{contrasts with Setting 1. Unlike Setting 1 in which the expected reward of arm 3 in objective 2 is much higher than the lexicographic optimal expected reward, in Setting 3, it is much lower. However, the gap of arm 3 in objective 2 in Setting 3 is same as the absolute gap of arm 3 in Setting 1.}
For all cases, we set $T=10^5$ and average the regret of the learners over 100 individual runs. We consider OM-LEX, NOM-LEX, and PF-LEX with prior knowledge and parameters that are summarized in Table~\ref{tbl:learners}. For PF-LEX, we do not consider the choices for $\epsilon$ and $\delta$ given in Theorem~\ref{theorem:gapindependent} because they require a large number of rounds for the initial exploration stage of the algorithm. Instead, we consider different exponents of $T$ as both $\epsilon$ and $\delta$ (see the appendix for additional discussion and a result which shows the regret of PF-LEX for $\epsilon = \delta = T^{-1/3}$ for $T=5\times10^8$).

\begin{table}
    \centering
    \begin{minipage}[b]{.48\linewidth}
        \caption{Expected reward vectors.}
        \label{tbl:settings}
        \smallskip
        \centering
        \begin{tabular}{lrrr}
            \toprule
            \textbf{Setting} & $\bm{\mu}_1$ & $\bm{\mu}_2$ & $\bm{\mu}_3$ \\
            \midrule
            Setting 1 & $\begin{pmatrix}0.50\\0.50\end{pmatrix}$ & $\begin{pmatrix}0.50\\0.40\end{pmatrix}$ & $\begin{pmatrix}0.40\\0.90\end{pmatrix}$ \\
            \addlinespace[3.7pt]
            Setting 2 & $\begin{pmatrix}0.50\\0.50\end{pmatrix}$ & $\begin{pmatrix}0.50\\0.40\end{pmatrix}$ & $\begin{pmatrix}0.40\\0.50\end{pmatrix}$ \\
            \addlinespace[3.7pt]
            Setting 3 & $\begin{pmatrix}0.50\\0.50\end{pmatrix}$ & $\begin{pmatrix}0.50\\0.40\end{pmatrix}$ & $\begin{pmatrix}0.40\\0.10\end{pmatrix}$ \\
            \bottomrule
        \end{tabular}
    \end{minipage}
    \begin{minipage}[b]{.48\linewidth}
        \caption{Prior knowledge and parameters of the algorithms.}
        \label{tbl:learners}
        \smallskip
        \centering
        \begin{tabular}{ll}
            \toprule
            \textbf{Algorithm} & \textbf{Prior Knowledge} \\
            & \textbf{\& Parameters} \\
            \midrule
            OM-LEX 1 & $\mu_*^1=\mu_*^2=0.50$ \\
            NOM-LEX 1 & $\eta_1=\eta_2=0.45$ \\
            NOM-LEX 2 & $\eta_1=\eta_2=0.40+10^{-6}$ \\
            NOM-LEX 3 & $\eta_1=\eta_2=0.50-10^{-6}$ \\
            PF-LEX 1 & $\epsilon=\delta=T^{-1/5}$ \\
            PF-LEX 2 & $\epsilon=\delta=T^{-1/10}$ \\
            \bottomrule
        \end{tabular}
    \end{minipage}
\end{table}

Table~\ref{tbl:results} shows the regrets of OM-LEX 1, NOM-LEX 1, 2, 3, and PF-LEX 1, 2 in Settings 1, 2 and 3 at $T=10^5$. There, we also report the performance of the variants of OM-LEX and NOM-LEX which only learn from the first objective and ignore the second objective, i.e., they act as if $D=1$. Note that all settings are equivalent for objective 1.

\begin{table}
    \caption{Regrets of OM-LEX 1, NOM-LEX 1, 2, and 3 in Settings 1, 2, and 3 along with their single objective variants (labeled as SO).\protect\footnotemark}
    \label{tbl:results}
    \smallskip
    \centering
    \setlength{\tabcolsep}{2.3pt}
    \begin{tabular}{lr@{\,}c@{\,}lr@{\,}c@{\,}lr@{\,}c@{\,}lr@{\,}c@{\,}lr@{\,}c@{\,}lr@{\,}c@{\,}l}
        \toprule
        & \multicolumn{6}{c}{\textbf{Setting 1}} & \multicolumn{6}{c}{\textbf{Setting 2}} & \multicolumn{6}{c}{\textbf{Setting 3}} \\
        \cmidrule(lr){2-7} \cmidrule(lr){8-13} \cmidrule(lr){14-19}
        \textbf{Learners} & \multicolumn{3}{c}{\textbf{Obj. 1}} & \multicolumn{3}{c}{\textbf{Obj. 2}} & \multicolumn{3}{c}{\textbf{Obj. 1}} & \multicolumn{3}{c}{\textbf{Obj. 2}} & \multicolumn{3}{c}{\textbf{Obj. 1}} & \multicolumn{3}{c}{\textbf{Obj. 2}} \\
        \midrule
        OM-LEX 1 & 12.0&$\pm$&2.1 & 333&$\pm$&56 & 321&$\pm$&71 & 314&$\pm$&61 & 11.0&$\pm$&2.0 & 323&$\pm$&60 \\
        OM-LEX 1 (\textsc{so}) & 334&$\pm$&73 & && & 334&$\pm$&73 & && & 334&$\pm$&73 \\
        NOM-LEX 1 & 1210&$\pm$&700 & 1150&$\pm$&680 & 4450&$\pm$&3800 & 2400&$\pm$&2900 & 285&$\pm$&110 & 270&$\pm$&120 \\
        NOM-LEX 2 & 1250&$\pm$&630 & 1320&$\pm$&600 & 1240&$\pm$&600 & 1160&$\pm$&660 & 14.9&$\pm$&12 & 4990&$\pm$&3000 \\
        NOM-LEX 3 & 12.7&$\pm$&7.0 & 1250&$\pm$&640 & 253&$\pm$&140 & 269&$\pm$&140 & 8.38&$\pm$&5.6 & 245&$\pm$&140 \\
        NOM-LEX 1 (\textsc{so}) & 706&$\pm$&770 & && & 706&$\pm$&770 & && & 706&$\pm$&770 \\
        PF-LEX 1 & 764&$\pm$&210 & 723&$\pm$&\rev{1.1p} & 806&$\pm$&240 & 723&$\pm$&\rev{1.1p} & 679&$\pm$&77 & 723&$\pm$&\rev{1.1p} \\
        PF-LEX 2 & 9820&$\pm$&4.5 & 52.8&$\pm$&\rev{14f} & 5000&$\pm$&860 & 94.6&$\pm$&24 & 52.8&$\pm$&\rev{14f} & 105&$\pm$&32 \\
        \bottomrule
    \end{tabular}
\end{table}
\footnotetext{p denotes $\times 10^{-12}$ and f denotes $\times 10^{-15}$. Average regrets are rounded to three most significant digits and standard deviations are rounded to two most significant digits.}

By looking at the regrets in objective 1 of OM-LEX 1, NOM-LEX 1, and their single objective variants, we observe how information from objective 2 helps learning in objective 1. OM-LEX takes advantage of large absolute gaps independent from whether the actual mean reward is higher or lower than the mean reward of arm 1. As a result, in Settings 1 and 3, it achieves lower regret in objective 1 than its single objective variant does. NOM-LEX is capable of doing this only when the gap is positive, a large absolute gap is not sufficient. As a result, only in Setting 3, it outperforms its single objective variant. \rev{In Setting 2, where information from objective 2 is not as useful as it is in Settings 1 and 3 to rule out the suboptimal arm in objective 1, OM-LEX 1 and NOM-LEX 1 perform worse than the other settings in objective 1.}

By looking at the regrets of NOM-LEX 1, 2, and 3, we observe how different prior information affects the performance of NOM-LEX. Consistent with the proven regret bounds, knowing near optimal expected rewards that are closer to the lexicographic optimal ones decreases the regret in all objectives. When the near optimal expected rewards are extremely close to the lexicographic optimal ones, the performance of NOM-LEX is very similar to that of OM-LEX.

\section{Conclusion and Future Work} \label{sec:conc}
We proposed a new multiobjective MAB problem, called the Lex-MAB, where the learner aims to minimize its regret with respect to lexicographic optimal arms. We studied the Lex-MAB with and without prior information on the expected rewards. We proved that the regret is uniformly bounded when the learner knows either the lexicographic optimal expected rewards or near-lexicographic optimal expected rewards. \reva{We also made a connection between knowing near-lexicographic optimal expected rewards and satisficing, and proved that uniformly bounded regret can be achieved for the MAB with satisficing objectives.} Interestingly, we showed that the regret incurred in an objective also depends on the prior information and observations from the other objectives. We also showed that the learner can still achieve $\tilde{O}(T^{2/3})$ gap-free regret without prior information. An important open question is whether $O(\sqrt{T})$ gap-free regret can be achieved in the Lex-MAB. This will answer the question of whether learning in the Lex-MAB without prior information is fundamentally more difficult than learning in the classical stochastic MAB without prior information. The case where there is prior information only for a subset of objectives is also worth investigating in the future. Another important future research direction is to extend the current work in the multiobjective MAB to minimize the regret with respect to a target arm other than a lexicogrpahic optimal arm by exploiting the prior information about the expected rewards of the target arm. 

\newpage

\bibliographystyle{IEEEtran}
\bibliography{references}

\newpage
\appendix
\section{Appendix}\label{sec:appendix}

\subsection{Additional Theorems}
\begin{theorem} \label{thm:OM-LEX-supp}
	When OM-LEX is run for Case 1, $\forall i\in\mathcal{D}$ and $\forall T\geq 1$, we have
	\begin{align*}
        \Ex[\Reg^i_{pf}(T)] \leq \sum_{a\in\mathcal{S}^i} \left( \rev{\left(\frac{\pi^2}{3}D + 1 \right)}\Delta_a^i+ \frac{36}{\nabla_a^{\max}}\log\frac{17}{\nabla_a^{\max}} \right) ~.
    \end{align*}
\end{theorem}

\begin{proof}
	Note that 
	\begin{align}
		\Ex[\Reg^i_{pf}(T)]  &= \Ex \left[ \sum_{t=1}^T \Delta_{a(t)}^i \right] \nonumber \\
		&\leq  \Ex \left[ \sum_{t=1}^T \Delta_{a(t)}^i\In\{a(t)\in\mathcal{S}^i\} \right] \nonumber \\
		&\leq \sum_{a\in\mathcal{S}^i} \left( \rev{\left(\frac{\pi^2}{3}D + 1 \right)} \Delta_a^i+ \frac{36}{\nabla_a^{\max}}\log\frac{17}{\nabla_a^{\max}} \right) \label{eqn:OM-LEX-supp}
	\end{align}
	where we prove \eqref{eqn:OM-LEX-supp} by replacing $\mathcal{S}_*^i$ with $\mathcal{S}^i$ in the proof of Theorem \ref{thm:OM-LEX}.
\end{proof}

\begin{theorem} \label{thm:NOM-LEX-supp}
	Redefine $\Delta_{\min}^i:=\min_{a\in\mathcal{S}^i}\Delta_a^i$, $\forall i\in\mathcal{D}$. When NOM-LEX is run for Case 2, $\forall i\in\mathcal{D}$ and $\forall T\geq 1$, we have
	\begin{align*}
		\Ex[\Reg^i_{pf}(T)] \leq \sum_{a\in\mathcal{S}^i} \left( \rev{\left(\frac{\pi^2}{6}D + 1 \right)}\Delta_a^i+ \frac{36\Delta_a^i}{\rev{(\max_{j\in\mathcal{D}}(\Delta_a^j-\delta_j))^2}}\log\frac{17}{\max_{j\in\mathcal{D}}(\Delta_a^j-\delta_j)} \right) ~.
	\end{align*}
\end{theorem}

\begin{proof}
	We redefine $\Delta_{\min}^i$ as stated in the theorem. Then, 
	\begin{align}
		\Ex [ \Reg^i_{pf}(T) ] &= \Ex \left[ \sum_{t=1}^T \Delta_{a(t)}^i \right] \nonumber \\
		&\leq \Ex \left[ \sum_{t=1}^T \Delta_{a(t)}^i\In\{a(t)\in\mathcal{S}^i\} \right] \nonumber \\
		&\leq \sum_{a\in\mathcal{S}^i} \left( \rev{\left(\frac{\pi^2}{6}D + 1 \right)}\Delta_a^i+ \frac{36\Delta_a^i}{\rev{(\max_{j\in\mathcal{D}}(\Delta_a^j-\delta_j))^2}}\log\frac{17}{\max_{j\in\mathcal{D}}(\Delta_a^j-\delta_j)} \right) \label{eqn:NOM-LEX-supp}
	\end{align}
	where we prove \eqref{eqn:NOM-LEX-supp} by replacing $\mathcal{S}_*^i$ with $\mathcal{S}^i$ in the proof of Theorem \ref{thm:NOM-LEX}.
\end{proof}

\rev{Note that redefining $\Delta_{\min}^i$ as $\min_{a\in\mathcal{S}^i}\Delta_a^i$ in Case 2 implies that the learner has stronger prior knowledge on the near-lexicographic optimal expected rewards, since $\min_{a\in\mathcal{S}^i}\Delta_a^i \leq \min_{a\in\mathcal{S}^i_*}\Delta_a^i$.}

\subsection{Multiobjective MAB with Satisficing Objectives} \label{sec:satisficing}

In this section, we extend the satisfaction-in-mean-reward problem introduced in \cite{reverdy2017satisficing} to the multiobjective setting. We keep the same system model but introduce the concept of satisficing optimality and a new notion of regret that captures this concept.

\textbf{Satisficing Optimality:} In the satisficing setting, the learner is given a target threshold $\eta_i$ for each objective $i\in\mathcal{D}$. We say that an arm $a$ is satisficing ``optimal'' or simply satisficing in objective $i$ if and only if its mean reward in objective $i$ is equal to or larger than the corresponding target threshold. Let $\mathcal{A}_s^i:=\{a\in\mathcal{A}:\mu_a^i\geq\eta_i\}$ be the set of satisficing arms in objective $i$ and $\mathcal{S}_s^i:=\mathcal{A}-\mathcal{A}_s^i$ be the set of non-satisficing arms in objective $i$. The satisficing goal is to play arms that are satisficing in all objectives. We assume such arms exist and call them satisficing ``optimal'' arms. Then, we use $*$ to denote an arbitrary satisficing ``optimal'' arm and call it the ``optimal'' satisficing arm. Note that $\eta_i\leq\mu_*^i$ for all $i\in\mathcal{D}$ and define $\delta_i:=\mu_*^i-\eta_i$ for all $i\in\mathcal{D}$.

\textbf{Regret Definition:} The satisficing regret in objective $i$ is given as $\Reg_s^i(T):=\sum_{t=1}^T (\Delta_{a(t)}^i-\delta_i) \mathbb{I} \{ a(t)\in\mathcal{S}_s^i \}$ and the satisficing regret is defined as the tuple $\bm{\Reg}_s(T):=(\Reg_s^1(T),\allowbreak\ldots,\allowbreak\Reg_s^D(T))$. Note that an arm $a$ incurs regret in objective $i$ only when it is not satisficing in that objective and the amount of regret incurred is equal to the gap between its mean reward in objective $i$ and the corresponding target threshold, i.e., \reva{$\Delta_a^i-\delta_i = \eta_i - \mu_{a}^i$}.

\begin{remark}
    When $D=1$, the multiobjective formulation reduces to the exact same problem introduced in \cite{reverdy2017satisficing} as Problems 1 and 2 (satisfaction-in-mean-reward problem).
\end{remark}

Assuming $\eta_i=0$ for all $i\in\mathcal{D}$ without any loss of generality,\footnote{See Remark \ref{rmk:normalization} for a detailed discussion of why the generality is not lost.} the algorithm proposed for Case 2, which is NOM-LEX, can also be used to solve the satisficing goal in the multiobjective MAB. Since the goal now is to minimize the satisficing regret rather than the lexicographic regret, \reva{we no longer need $\eta_i$ to lie between the lexicographic optimal and the second highest lexicographic optimal expected rewards in objective $i$.} The following theorem bounds the expected satisficing regret for NOM-LEX.

\begin{theorem} \label{thm:satisficing}
    When NOM-LEX is run for the satisficing goal, $\forall i\in\mathcal{D}$ and $\forall T\geq 1$, we have
    \begin{align*}
        \Ex[\Reg_s^i(T)] \leq \sum_{a\in\mathcal{S}_s^i}\left( \left(\frac{\pi^2}{6}D+1\right)(\Delta_a^i-\delta_i) + \frac{36}{\max_{j\in\mathcal{D}}(\Delta_a^j-\delta_j)}\log\frac{17}{\max_{j\in\mathcal{D}}(\Delta_a^j-\delta_j)} \right) ~.
    \end{align*}
\end{theorem}

\begin{proof}
    Replacing lexicographic optimality with satisficing optimality, $\mathcal{S}_*^i$ with \reva{$\mathcal{S}_s^i$}, and every instance of the exact phrase $\Delta_a^i$ (and $\Delta_{a(t)}^i$) with $\Delta_a^i-\delta_i$ (and $\Delta_{a(t)}^i-\delta_i$), the proof of Theorem \ref{thm:NOM-LEX} holds for Theorem \ref{thm:satisficing}
     as well. Note that
    \begin{align*}
        \frac{\Delta_a^i-\delta_i}{(\max_{j\in\mathcal{D}}(\Delta_a^j-\delta_j))^2} \leq \frac{1}{\max_{j\in\mathcal{D}}(\Delta_a^j-\delta_j)}
    \end{align*}
    for all $i\in\mathcal{D}$.
\end{proof}

\begin{corollary} \label{crl:satisficing}
    When NOM-LEX is run for the single objective satisfaction-in-mean-reward problem ($D=1$), $\forall T\geq 1$, we have
    \begin{align*}
        \Ex[\Reg_s(T)] \leq \sum_{a\in\mathcal{S}_s}\left( \left(\frac{\pi^2}{6}D+1\right)(\Delta_a-\delta) + \frac{36}{\Delta_a-\delta}\log\frac{17}{\Delta_a-\delta} \right) ~.\footnotemark
    \end{align*}
    \footnotetext{For simplicity, the objective index 1 is omitted.}
\end{corollary}

\begin{remark}
    Theorem \ref{thm:satisficing} and Corollary \ref{crl:satisficing} show that bounded regret is possible for satisfaction-in-mean-reward problem whether it is single objective or multiobjective, when the learner is given the target thresholds. This result is directly in conflict with Corollary 2 of \cite{reverdy2017satisficing}, which claims a logarithmic lower bound on the single objective case, and suggests that satisfaction-in-mean-reward UCL algorithm given in Section VI-A of \cite{reverdy2017satisficing} is not optimal since it fails to achieve bounded regret.
\end{remark}

\subsection{Additional Experiments}
Figure~\ref{fig:general} shows the regrets of OM-LEX 1, NOM-LEX 1, PF-LEX 1 in Setting 1. We observe that the regret of OM-LEX in objective 1 is significantly smaller than the regret of NOM-LEX. We believe this is the case because OM-LEX is able to take advantage of the large absolute gap of arm 3 to eliminate it early on, whereas NOM-LEX cannot. The behavior of PF-LEX is explained as follows. Until around round 30000, it explores all three arms uniformly since their estimated rewards in objective 1 are still chained to each other. At round 30000, the gap between the arms is deemed small enough with respect to the time horizon of the problem.
In the remaining rounds, it plays only the optimistic near-lexicographic optimal arm in objective $2$ ($\hat{a}^2_*(t)$). 
For this case, $\epsilon$ matches with the minimum suboptimality gap. Thus, although PF-LEX always chooses $\hat{a}^2_*(t)$, because $\hat{{\cal A}}^1_*(t) = {\cal A}^1_*$, it learns to play optimally. 
\rev{As a remark, we note that PF-LEX could incur high regret in objective $1$ (see PF-LEX 2 in Table \ref{tbl:results}) if the minimum suboptimality gap were smaller than $\epsilon$. }

\begin{figure}
	\centering
	\includegraphics[width=.75\linewidth]{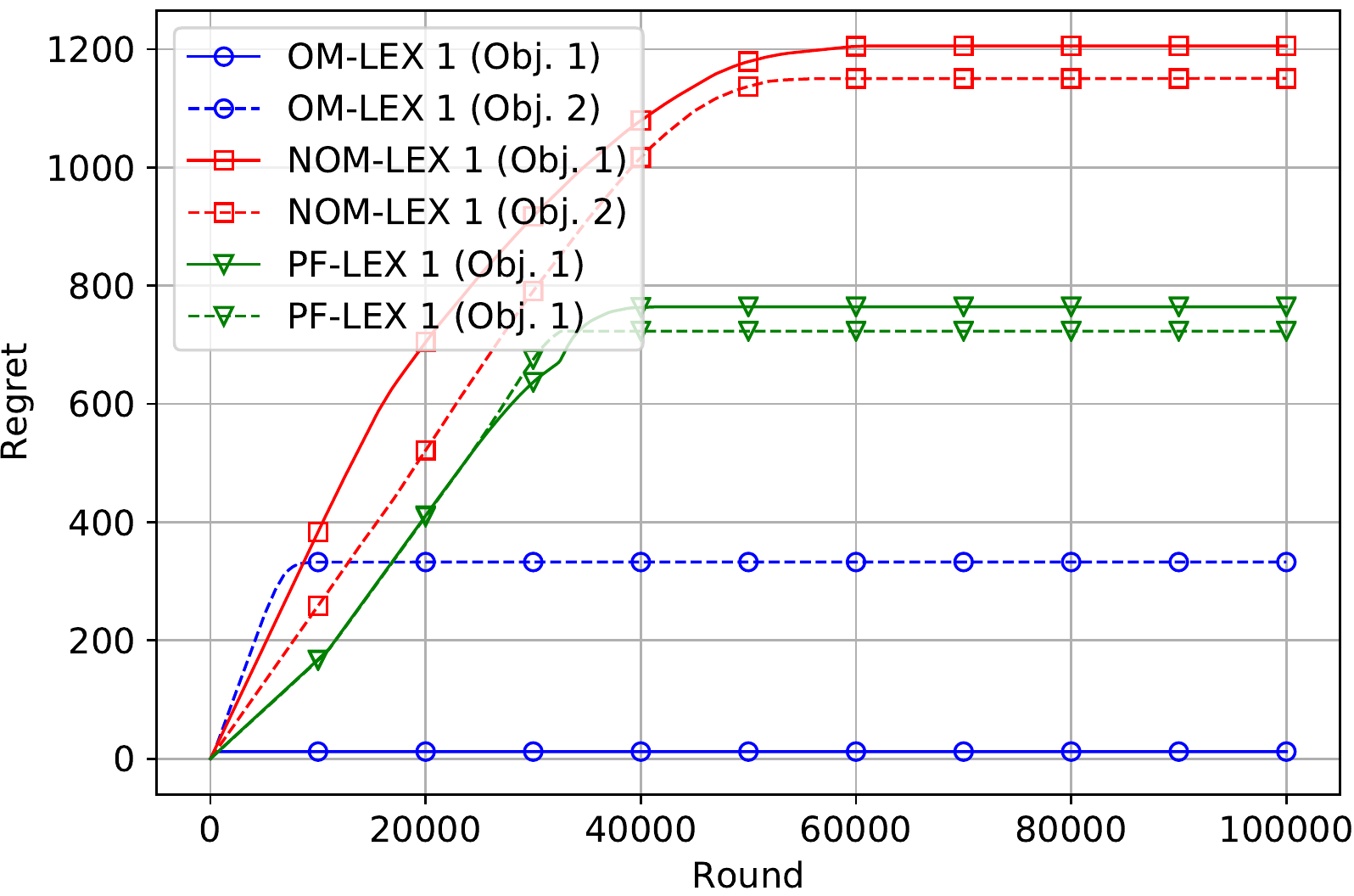}
	\caption{Regrets of OM-LEX 1, NOM-LEX 1, and PF-LEX 1 in Setting 1.}
	\label{fig:general}
\end{figure}

Next, for Setting 1, we consider PF-LEX 3 that has parameters $\epsilon=\delta=T^{-1/3}$ as given in Table~\ref{tbl:learners2} (that match with the optimal choice for $\epsilon$ given in Theorem \ref{theorem:gapindependent}), run simulations for $T=5\times10^8$, and report the average regret of the learner over 5 runs (Figure~\ref{fig:general-case3}). This result illustrates the identifiability problem introduced earlier that makes learning lexicographic optimal arms particularly challenging. We see that PF-LEX rules out arm 3 as a potential lexicographic optimal arm and stops incuring regret in objective 1 very early on. However, since it is not possible to be confident in that both arm 1 and arm 2 have equal expected rewards in objective 1, the algorithm still keeps exploring them uniformly until around round $2\times 10^8$. During this exploration stage, it incurs linear regret in objective 2. Once both arms are deemed to be optimal in objective 1, PF-LEX starts exploiting the optimistic near-lexicographic optimal arm in objective 2, after which the increase of the regret in objective 2 drops drastically.

\begin{figure}
	\centering
	\includegraphics[width=.75\linewidth]{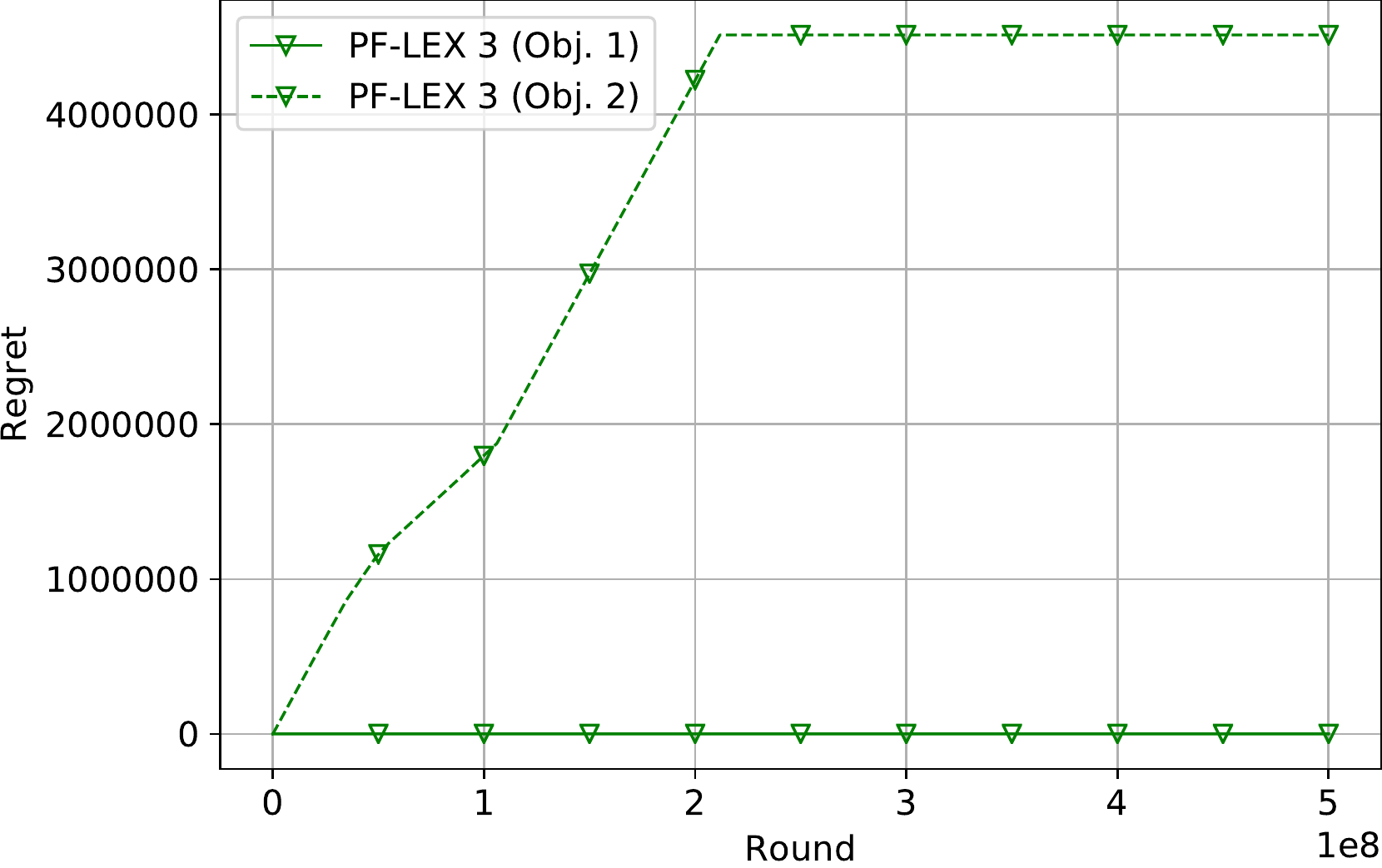}
	\caption{Regret of PF-LEX 3 in Setting 1.}
	\label{fig:general-case3}
\end{figure}

We also present experiments in two additional settings with $D=3$. In Setting 4, there are 43 arms with expected reward vectors in $\{0.90,0.50,0.40,0.10\}^3$ such that each arm has a \red{unique expected reward vector}, where we eliminated arms that lexicographically dominate $(0.50,0.50,0.50)$ so that it is the only lexicographic optimal arm and $\Delta_{\min}^1=\Delta_{\min}^2=\Delta_{\min}^3=0.10$. Setting 4 features a large variety of arms with combinations of expected rewards that are much higher than, equal to, slightly lower than, and much lower than the lexicographic optimal expected rewards in all objectives. 
In Setting 5, there are 19 arms, where we eliminated arms in $\mathcal{S}_*^2$ from Setting 4 so that $\mathcal{S}_*^2=\emptyset$ and $\Delta_{\min}^2=\infty$ while $\Delta_{\min}^1=\Delta_{\min}^3=0.10$.
Table \ref{tbl:learners2} shows OMG-LEX and NOM-LEX with different prior knowledge and parameters than the ones considered in Section \ref{sec:experiments}.

\begin{table}
	\caption{Prior knowledge and parameters of the algorithms for additional experiments.}
	\label{tbl:learners2}
	\smallskip
	\centering
	\begin{tabular}{ll}
		\toprule
		\textbf{Algorithm} & \textbf{Prior Knowledge \& Parameters} \\
		\midrule
		OM-LEX 2 & $\mu_*^1=\mu_*^2=\mu_*^3=0.50$ \\
		NOM-LEX 4 & $\eta_1=\eta_2=\eta_3=0.45$ \\
		NOM-LEX 5 & $\eta_1=\eta_3=0.45$, $\eta_2=-10^6$ \\
		PF-LEX 3 & $\epsilon=\delta=T^{-1/3}$ \\
		\bottomrule
	\end{tabular}
\end{table}

We run simulations with $T=10^5$ and average the regret of the learners over 100 individual runs. Different from simulations in Section \ref{sec:experiments}, we give results for priority-free regrets as well.

Table~\ref{tbl:results2} shows the priority-based and priority-free regrets of OM-LEX 2, NOM-LEX 4 and 5 in Settings 4 and 5 at $T=10^5$. Since NOM-LEX considers arms with very high expected rewards compared to the near optimal expected reward as potential optimal arms, it tends to incur a lot more negative regret in priority-free settings as opposed to OM-LEX, which only looks for arms with expected rewards that are very close to the lexicographic optimal expected rewards.

\rev{In Setting 5,\footnote{Note that the prior information of NOM-LEX 5 is not valid for Setting 4.} note that any $\delta_2>0$ would guarantee a bounded regret for NOM-LEX (since $\Delta_{\min}^2=\infty$).} Moreover, $\delta_2$ appears in none of our regret bounds in Theorems \ref{thm:NOM-LEX} and \ref{thm:NOM-LEX-supp} for Setting 5. However, our numerical experiments show that it still affects the regret. This is because, knowing a larger $\delta_2$ better captures the information $\Delta_{\min}^2=\infty$ and results in having smaller regret in objective 1. 

\begin{table}
    \caption{Priority-based and priority-free regrets of OM-LEX 2, NOM-LEX 4 and 5 in Settings 4 and 5.}
    \label{tbl:results2}
    \smallskip
    \centering
    \setlength{\tabcolsep}{2.3pt}
    \begin{tabular}{lr@{\,}c@{\,}lr@{\,}c@{\,}lr@{\,}c@{\,}lr@{\,}c@{\,}lr@{\,}c@{\,}lr@{\,}c@{\,}l}
        \toprule
        & \multicolumn{9}{c}{\textbf{Setting 4}} & \multicolumn{9}{c}{\textbf{Setting 5}} \\
        \cmidrule(lr){2-10} \cmidrule(lr){11-19}
        \textbf{Learners} & \multicolumn{3}{c}{\textbf{Obj. 1}} & \multicolumn{3}{c}{\textbf{Obj. 2}} & \multicolumn{3}{c}{\textbf{Obj. 3}} & \multicolumn{3}{c}{\textbf{Obj. 1}} & \multicolumn{3}{c}{\textbf{Obj. 2}} & \multicolumn{3}{c}{\textbf{Obj. 3}} \\
		\midrule
		\multicolumn{19}{l}{OM-LEX 2} \\[1.3pt]
		\textit{pr.-based} & 2000&$\pm$&100 & 821&$\pm$&75 & 367&$\pm$&59 & 1010&$\pm$&82 & && & 373&$\pm$&72 \\
		\textit{pr.-free} & 1990&$\pm$&120 & 1440&$\pm$&110 & 1290&$\pm$&110 & 1040&$\pm$&81 & -350&$\pm$&17 & -350&$\pm$&17 \\
		\addlinespace
		\multicolumn{19}{l}{NOM-LEX 4} \\[1.3pt]
        \textit{pr.-based} & 6620&$\pm$&2000 & 2160&$\pm$&800 & 684&$\pm$&420 & 7180&$\pm$&2000 & && & 1160 &$\pm$& 470 \\
		\textit{pr.-free} & 6840&$\pm$&1800 & -4490&$\pm$&3000 & -7910&$\pm$&3200 & 7250&$\pm$&2200 & -14100&$\pm$&4300 & -5930&$\pm$&2500 \\
		\addlinespace
		\multicolumn{19}{l}{NOM-LEX 5} \\[1.3pt]
        \textit{pr.-based} & && & && & && & 6570&$\pm$&2700 & && & 1020&$\pm$&550 \\
        \textit{pr.-free} & && & && & && & 6670&$\pm$&2600 & -12900&$\pm$&5200 & -5860&$\pm$&2800 \\
        \bottomrule
    \end{tabular}
\end{table}

\subsection{Proof of Theorem \ref{thm:OM-LEX}}

We use the following fact to prove both Theorem \ref{thm:OM-LEX} and \ref{thm:NOM-LEX}.
\begin{fact}[Results from the proof of Theorem 9 in \citet{garivier}] \label{fact:garivier}
    Given $\Delta>0$, for all arms $a\in\mathcal{A}$, for all objectives $i\in\mathcal{D}$ and for all rounds $t\in\{1,2,\ldots,T\}$, we have
    \begin{align*}
        &\sum_{w=1}^{\infty} \Pr\left( \hat{\mu}_a^i(t)-\mu_a^i>\Delta-\sqrt{\frac{4\log N_a(t)}{N_a(t)}} \;\middle|\; N_a(t)=w  \right)  \\
        &\quad = \sum_{w=1}^{\infty} \Pr\left( \hat{\mu}_a^i(t)-\mu_a^i<\sqrt{\frac{4\log N_a(t)}{N_a(t)}}-\Delta \;\middle|\; N_a(t)=w \right) \\
        &\quad \leq \frac{36}{\Delta^2}\log\frac{17}{\Delta} ~.
    \end{align*}
\end{fact}

For an arm $a$ that is not lexicographic optimal, let $\dagger(a):=\argmax_{j\in\mathcal{D}}\nabla_a^j$ so that $\nabla_a^{\dagger(a)}=\nabla_a^{\max}$. When $a$ can be inferred from the context, $\dagger(a)$ is denoted by $\dagger$ only. For all objectives $i\in\mathcal{D}$, we decompose $\Ex[\Reg^i(T)]$ as
\begin{align}
    \Ex[\Reg^i(T)] &= \Ex\left[ \sum_{t=1}^T \In\{a(t)\in\mathcal{S}_*^i\} \Delta_{a(t)}^i \right] \nonumber \\
    &= \Ex\left[ \sum_{a\in\mathcal{S}_*^i} \sum_{t=1}^T \In\{a(t)=a\} \Delta_a^i \right] \nonumber \\
    &= \Ex\left[ \sum_{a\in\mathcal{S}_*^i} \sum_{t=1}^T \In\{t\leq A, a(t)=a\} \Delta_a^i \right] \label{eqn:case1-decomp1} \\
    &\quad + \Ex\left[ \sum_{a\in\mathcal{S}_*^i} \sum_{t=1}^T \In\left\{t>A, |\hat{\mu}_a^\dagger(t)|<\sqrt{\frac{4\log N_a(t)}{N_a(t)}}, a(t)=a\right\} \Delta_a^i \right] \label{eqn:case1-decomp2} \\
    &\quad + \Ex\left[ \sum_{a\in\mathcal{S}_*^i} \sum_{t=1}^T \In\left\{t>A, |\hat{\mu}_a^\dagger(t)|\geq\sqrt{\frac{4\log N_a(t)}{N_a(t)}}, a(t)=a\right\} \Delta_a^i \right] \label{eqn:case1-decomp3} ~.
\end{align}

Bounding \eqref{eqn:case1-decomp1} is trivial, since each arm is played exactly once for rounds $t\leq A$. We have
\begin{align*}
    \Ex\left[ \sum_{a\in\mathcal{S}_*^i} \sum_{t=1}^T \In\{t\leq A, a(t)=a\} \Delta_a^i \right] &= \Ex\left[ \sum_{a\in\mathcal{S}_*^i} \sum_{t=1}^T \In\{t=a\} \Delta_a^i \right] \\
    &= \sum_{a\in\mathcal{S}_*^i} \Ex\left[ \sum_{t=1}^T \In\{t=a\} \right] \Delta_a^i \\
    &= \sum_{a\in\mathcal{S}_*^i} \Delta_a^i ~.
\end{align*}

In order to bound \eqref{eqn:case1-decomp2}, we define $\tau_a(w)$ as the $w$th round for which $a(t)=a$, and $w_a(T)$ as the number of rounds for which $a(t)=a$ by round $T$. By definition, $N_a(\tau_a(w+1))=w$, \rev{$\tau_a(1)\leq A$} and $w_a(T)\leq T$ hold for all arms $a$. Thus, we have
\begin{align}
    &\Ex\left[ \sum_{a\in\mathcal{S}_*^i} \sum_{t=1}^T \In\left\{t>A, |\hat{\mu}_a^\dagger(t)|<\sqrt{\frac{4\log N_a(t)}{N_a(t)}}, a(t)=a\right\} \Delta_a^i \right] \nonumber \\
    &\quad \rev{=} \sum_{a\in\mathcal{S}_*^i} \Ex\left[ \sum_{w=1}^{\rev{w_a(T)-1}} \sum_{t=\tau_a(w)+1}^{\tau_a(w+1)} \In\left\{ |\hat{\mu}_a^\dagger(t)|<\sqrt{\frac{4\log N_a(t)}{N_a(t)}}, a(t)=a\right\}  \right] \Delta_a^i \nonumber \\
    &\quad = \sum_{a\in\mathcal{S}_*^i} \Ex\left[ \sum_{w=1}^{\rev{w_a(T)-1}} \In\left\{ |\hat{\mu}_a^\dagger(\tau_a(w+1))|<\sqrt{\frac{4\log w}{w}} \right\}  \right] \Delta_a^i \nonumber \\
    &\quad \leq \sum_{a\in\mathcal{S}_*^i} \sum_{w=1}^{\infty} \Pr\left( |\hat{\mu}_a^\dagger(\tau_a(w+1))|<\sqrt{\frac{4\log w}{w}} \right) \Delta_a^i \label{eqn:case1-decomp2c} ~.
\end{align}

When $\mu_a^\dagger=\nabla_a^\dagger$, we have
\begin{align}
    \sum_{w=1}^{\infty} \Pr\left( |\hat{\mu}_a^\dagger(\tau_a(w+1))|<\sqrt{\frac{4\log w}{w}} \right) &\leq \sum_{w=1}^{\infty} \Pr\left( \hat{\mu}_a^\dagger(\tau_a(w+1))<\sqrt{\frac{4\log w}{w}} \right) \nonumber \\
    &\leq \sum_{w=1}^{\infty} \Pr\left( \hat{\mu}_a^\dagger(\tau_a(w+1))-\mu_a^\dagger<\sqrt{\frac{4\log w}{w}}-\nabla_a^\dagger \right) \nonumber \\
    &\leq \frac{36}{(\nabla_a^\dagger)^2}\log\frac{17}{\nabla_a^\dagger} \label{eqn:case1-decomp2a} ~,
\end{align}
where \eqref{eqn:case1-decomp2a} is due to Fact \ref{fact:garivier}.

Similarly, when $\mu_a^\dagger=-\nabla_a^\dagger$, we have
\begin{align}
    \sum_{w=1}^{\infty} \Pr\left( |\hat{\mu}_a^\dagger(\tau_a(w+1))|<\sqrt{\frac{4\log w}{w}} \right) &\leq \sum_{w=1}^{\infty} \Pr\left( \hat{\mu}_a^\dagger(\tau_a(w+1))>-\sqrt{\frac{4\log w}{w}} \right) \nonumber \\
    &\leq \sum_{w=1}^{\infty} \Pr\left( \hat{\mu}_a^\dagger(\tau_a(w+1))-\mu_a^\dagger>\nabla_a^\dagger-\sqrt{\frac{4\log w}{w}} \right) \nonumber \\
    &\leq \frac{36}{(\nabla_a^\dagger)^2}\log\frac{17}{\nabla_a^\dagger} \label{eqn:case1-decomp2b} ~,
\end{align}
where \eqref{eqn:case1-decomp2b} is again due to Fact \ref{fact:garivier}.

Combining \eqref{eqn:case1-decomp2a} and \eqref{eqn:case1-decomp2b}, we obtain
\begin{align*}
    \eqref{eqn:case1-decomp2c} &\leq \sum_{a\in\mathcal{S}_*^i} \left(\frac{36}{(\nabla_a^\dagger)^2}\log\frac{17}{\nabla_a^\dagger}\right) \Delta_a^i \\
    &\leq \sum_{a\in\mathcal{S}_*^i} \frac{36}{\nabla_a^\dagger}\log\frac{17}{\nabla_a^\dagger} ~.
\end{align*}

In order to bound \eqref{eqn:case1-decomp3}, we observe that $t>A \wedge |\hat{\mu}_a^\dagger(t)|\geq\sqrt{4\log N_a(t)/N_a(t)} \wedge a(t)=a$ can only occur during an exploration stage, where each arm is played successively. Hence, we can infer that
\begin{enumerate}[label=(\roman*)]
    \item $a(t-a+*)=*$,
    \item $t-a+1>A$,
    \item $\hat{\mathcal{A}}_*(t-a+1)=\emptyset$, which implies that there exists an objective $j$ such that
    \begin{align*}
        |\hat{\mu}_*^j(t-a+*)| = |\hat{\mu}_*^j(t-a+1)| \geq \sqrt{4\frac{\log N_*(t-a+1)}{N_*(t-a+1)}} = \sqrt{4\frac{\log N_*(t-a+*)}{N_*(t-a+*)}} ~,
    \end{align*}
    since arm $*$ is not played after round $t-a+1$ until round $t-a+*$.
\end{enumerate}

Using these observations and defining $t_a:=t-a+*$, we obtain
\begin{align}
    &\Ex\left[ \sum_{a\in\mathcal{S}_*^i} \sum_{t=1}^T \In\left\{t>A, |\hat{\mu}_a^\dagger(t)|\geq\sqrt{\frac{4\log N_a(t)}{N_a(t)}}, a(t)=a\right\} \Delta_a^i \right] \nonumber \\
    &\quad = \Ex\left[ \sum_{a\in\mathcal{S}_*^i} \sum_{t=1}^T \In\left\{\exists j: t_a>A, |\hat{\mu}_*^j(t_a)|\geq\sqrt{\frac{4\log N_*(t_a)}{N_*(t_a)}}, a(t_a)=*\right\} \Delta_a^i \right] \nonumber \\
    &\quad \leq \sum_{a\in\mathcal{S}_*^i} \sum_{j=1}^D \Ex\left[ \sum_{t=1}^T \In\left\{t_a>A, |\hat{\mu}_*^j(t_a)|\geq\sqrt{\frac{4\log N_*(t_a)}{N_*(t_a)}}, a(t_a)=*\right\} \Delta_a^i \right] \nonumber \\
    &\quad \leq \sum_{a\in\mathcal{S}_*^i} \sum_{j=1}^D \Ex\left[ \sum_{t=1}^T \In\left\{t>A, |\hat{\mu}_*^j(t)|\geq\sqrt{\frac{4\log N_*(t)}{N_*(t)}}, a(t)=*\right\} \Delta_a^i \right] \nonumber \\
    &\quad \rev{=} \sum_{a\in\mathcal{S}_*^i} \sum_{j=1}^D \Ex\left[ \sum_{w=1}^{\rev{w_*(T)-1}} \sum_{t=\tau_*(w)+1}^{\tau_*(w+1)} \In\left\{ |\hat{\mu}_*^j(t)|\geq\sqrt{\frac{4\log N_*(t)}{N_*(t)}}, a(t)=*\right\}  \right] \Delta_a^i \nonumber \\
    &\quad = \sum_{a\in\mathcal{S}_*^i} \sum_{j=1}^D \Ex\left[ \sum_{w=1}^{\rev{w_*(T)-1}} \In\left\{ |\hat{\mu}_*^j(\tau_*(w+1))|\geq\sqrt{\frac{4\log w}{w}} \right\}  \right] \Delta_a^i \nonumber \\
    &\quad \leq \sum_{a\in\mathcal{S}_*^i} \sum_{j=1}^D \sum_{w=1}^{\infty} \Pr\left( |\hat{\mu}_*^j(\tau_*(w+1))|\geq\sqrt{\frac{4\log w}{w}} \right) \Delta_a^i \nonumber \\
    &\quad \rev{=} \sum_{a\in\mathcal{S}_*^i} \sum_{j=1}^D \sum_{w=1}^{\infty}\left[ \Pr\left( \hat{\mu}_*^j(\tau_*(w+1))\leq-\sqrt{\frac{4\log w}{w}} \right) + \Pr\left( \hat{\mu}_*^j(\tau_*(w+1))\geq\sqrt{\frac{4\log w}{w}} \right) \right] \Delta_a^i \nonumber \\
    &\quad \leq \sum_{a\in\mathcal{S}_*^i} \sum_{j=1}^D \sum_{w=1}^{\infty} \frac{2}{w^2} \Delta_a^i \label{eqn:case1-decomp3a} \\
    &\quad \leq \sum_{a\in\mathcal{S}_*^i} \rev{\left( \frac{\pi^2}{3} \right)} D\Delta_a^i \nonumber ~,
\end{align}
where \eqref{eqn:case1-decomp3a} is due to Hoeffding's \rev{inequality for sub-Gaussian random variables \cite{bubeck2013bounded}}.

\subsection{Proof of Theorem \ref{thm:NOM-LEX}}
For an arm $a$ that is not lexicographic optimal, let $\dagger(a):=\argmax_{j\in\mathcal{D}}\Delta_a^j-\delta_j$. When $a$ can be inferred from the context, $\dagger(a)$ is denoted by $\dagger$ only. For all objectives $i\in\mathcal{D}$, we decompose $\Ex[\Reg^i(T)]$ as
\begin{align}
    \Ex[\Reg^i(T)] &= \Ex\left[ \sum_{t=1}^T \In\{a(t)\in\mathcal{S}_*^i\} \Delta_{a(t)}^i \right] \nonumber \\
    &= \Ex\left[ \sum_{a\in\mathcal{S}_*^i} \sum_{t=1}^T \In\{a(t)=a\} \Delta_a^i \right] \nonumber \\
    &= \Ex\left[ \sum_{a\in\mathcal{S}_*^i} \sum_{t=1}^T \In\{t\leq A, a(t)=a\} \Delta_a^i \right] \label{eqn:case2-decomp1} \\
    &\quad + \Ex\left[ \sum_{a\in\mathcal{S}_*^i} \sum_{t=1}^T \In\left\{t>A, \hat{\mu}_a^\dagger(t)>-\sqrt{\frac{4\log N_a(t)}{N_a(t)}}, a(t)=a\right\} \Delta_a^i \right] \label{eqn:case2-decomp2} \\
    &\quad + \Ex\left[ \sum_{a\in\mathcal{S}_*^i} \sum_{t=1}^T \In\left\{t>A, \hat{\mu}_a^\dagger(t)\leq-\sqrt{\frac{4\log N_a(t)}{N_a(t)}}, a(t)=a\right\} \Delta_a^i \right] \label{eqn:case2-decomp3} ~.
\end{align}

Bounding \eqref{eqn:case2-decomp1} is trivial, since each arm is played exactly once for rounds $t\leq A$. We have
\begin{align*}
    \Ex\left[ \sum_{a\in\mathcal{S}_*^i} \sum_{t=1}^T \In\{t\leq A, a(t)=a\} \Delta_a^i \right] &= \Ex\left[ \sum_{a\in\mathcal{S}_*^i} \sum_{t=1}^T \In\{t=a\} \Delta_a^i \right] \\
    &= \sum_{a\in\mathcal{S}_*^i} \Ex\left[ \sum_{t=1}^T \In\{t=a\} \right] \Delta_a^i \\
    &= \sum_{a\in\mathcal{S}_*^i} \Delta_a^i ~.
\end{align*}

In order to bound \rev{\eqref{eqn:case2-decomp2}}, we use $\tau_a(w)$ and $w_a(T)$ defined in the proof of Theorem~\ref{thm:OM-LEX}. We have
\begin{align}
    &\Ex\left[ \sum_{a\in\mathcal{S}_*^i} \sum_{t=1}^T \In\left\{t>A, \hat{\mu}_a^\dagger(t)>-\sqrt{\frac{4\log N_a(t)}{N_a(t)}}, a(t)=a\right\} \Delta_a^i \right] \nonumber \\
    &\quad \rev{=} \sum_{a\in\mathcal{S}_*^i} \Ex\left[ \sum_{w=1}^{\rev{w_a(T)-1}} \sum_{t=\tau_a(w)+1}^{\tau_a(w+1)} \In\left\{ \hat{\mu}_a^\dagger(t)>-\sqrt{\frac{4\log N_a(t)}{N_a(t)}}, a(t)=a\right\}  \right] \Delta_a^i \nonumber \\
    &\quad = \sum_{a\in\mathcal{S}_*^i} \Ex\left[ \sum_{w=1}^{\rev{w_a(T)-1}} \In\left\{ \hat{\mu}_a^\dagger(\tau_a(w+1))>-\sqrt{\frac{4\log w}{w}} \right\}  \right] \Delta_a^i \nonumber \\
    &\quad \leq \sum_{a\in\mathcal{S}_*^i} \sum_{w=1}^{\infty} \Pr\left( \hat{\mu}_a^\dagger(\tau_a(w+1))>-\sqrt{\frac{4\log w}{w}} \right) \Delta_a^i \nonumber \\
    &\quad \leq \sum_{a\in\mathcal{S}_*^i} \sum_{w=1}^{\infty} \Pr\left( \hat{\mu}_a^\dagger(\tau_a(w+1))-\mu_a^\dagger>\Delta_a^\dagger-\delta_{\dagger(a)}-\sqrt{\frac{4\log w}{w}} \right) \Delta_a^i \nonumber \\
    &\quad \leq \sum_{a\in\mathcal{S}_*^i} \frac{36\Delta_a^i}{(\Delta_a^\dagger-\delta_{\dagger(a)})^2}\log\frac{17}{\Delta_a^\dagger-\delta_{\dagger(a)}} \label{eqn:case2-decomp2a} ~,
\end{align}
where \eqref{eqn:case2-decomp2a} is due to Fact \ref{fact:garivier}.

In order to bound \rev{\eqref{eqn:case2-decomp3}}, we observe that $t>A \wedge \hat{\mu}_a^\dagger(t) \rev{\leq} -\sqrt{4\log N_a(t)/N_a(t)} \wedge a(t)=a$ can only occur during an exploration stage, where each arm is played successively. Hence we can infer that
\begin{enumerate}[label=(\roman*)]
    \item $a(t-a+*)=*$,
    \item $t-a+1>A$,
    \item $\hat{\mathcal{A}}_*(t-a+1)=\emptyset$, which implies that there exists an objective $j$ such that
    \begin{align*}
        \hat{\mu}_*^j(t-a+*) = \hat{\mu}_*^j(t-a+1) \leq -\sqrt{4\frac{\log N_*(t-a+1)}{N_*(t-a+1)}} = -\sqrt{4\frac{\log N_*(t-a+*)}{N_*(t-a+*)}} ~,
    \end{align*}
    since arm $*$ is not played after round $t-a+1$ until round $t-a+*$.
\end{enumerate}

Using these observations and defining $t_a:=t-a+*$, we obtain
\begin{align}
    &\Ex\left[ \sum_{a\in\mathcal{S}_*^i} \sum_{t=1}^T \In\left\{t>A, \hat{\mu}_a^\dagger(t)\leq-\sqrt{\frac{4\log N_a(t)}{N_a(t)}}, a(t)=a\right\} \Delta_a^i \right] \nonumber \\
    &\quad = \Ex\left[ \sum_{a\in\mathcal{S}_*^i} \sum_{t=1}^T \In\left\{\exists j: t_a>A, \hat{\mu}_*^j(t_a)\leq-\sqrt{\frac{4\log N_*(t_a)}{N_*(t_a)}}, a(t_a)=*\right\} \Delta_a^i \right] \nonumber \\
    &\quad \leq \sum_{a\in\mathcal{S}_*^i} \sum_{j=1}^D \Ex\left[ \sum_{t=1}^T \In\left\{t_a>A, \hat{\mu}_*^j(t_a)\leq-\sqrt{\frac{4\log N_*(t_a)}{N_*(t_a)}}, a(t_a)=*\right\} \Delta_a^i \right] \nonumber \\
    &\quad \leq \sum_{a\in\mathcal{S}_*^i} \sum_{j=1}^D \Ex\left[ \sum_{t=1}^T \In\left\{t>A, \hat{\mu}_*^j(t)\leq-\sqrt{\frac{4\log N_*(t)}{N_*(t)}}, a(t)=*\right\} \Delta_a^i \right] \nonumber \\
    &\quad \rev{=} \sum_{a\in\mathcal{S}_*^i} \sum_{j=1}^D \Ex\left[ \sum_{w=1}^{\rev{w_*(T)-1}} \sum_{t=\tau_*(w)+1}^{\tau_*(w+1)} \In\left\{ \hat{\mu}_*^j(t)\leq-\sqrt{\frac{4\log N_*(t)}{N_*(t)}}, a(t)=*\right\}  \right] \Delta_a^i \nonumber \\
    &\quad = \sum_{a\in\mathcal{S}_*^i} \sum_{j=1}^D \Ex\left[ \sum_{w=1}^{\rev{w_*(T)-1}} \In\left\{ \hat{\mu}_*^j(\tau_*(w+1))\leq-\sqrt{\frac{4\log w}{w}} \right\}  \right] \Delta_a^i \nonumber \\
    &\quad \leq \sum_{a\in\mathcal{S}_*^i} \sum_{j=1}^D \sum_{w=1}^{\infty} \Pr\left( \hat{\mu}_*^j(\tau_*(w+1))\leq-\sqrt{\frac{4\log w}{w}} \right) \Delta_a^i \nonumber \\
    &\quad \leq \sum_{a\in\mathcal{S}_*^i} \sum_{j=1}^D \sum_{w=1}^{\infty} \Pr\left( \hat{\mu}_*^j(\tau_*(w+1))-\mu_*^j\leq-\sqrt{\frac{4\log w}{w}} \right) \Delta_a^i \label{eqn:case2-decomp3a} \\
    &\quad \leq \sum_{a\in\mathcal{S}_*^i} \sum_{j=1}^D \sum_{w=1}^{\infty} \frac{1}{w^2} \Delta_a^i \label{eqn:case2-decomp3b} \\
    &\quad \leq \sum_{a\in\mathcal{S}_*^i} \rev{\left( \frac{\pi^2}{6} \right)} D\Delta_a^i \nonumber ~,
\end{align}
where \eqref{eqn:case2-decomp3a} holds since $\mu_*^j=\delta_j \geq 0$ and \eqref{eqn:case2-decomp3b} is due to Hoeffding's inequality \rev{for sub-Gaussian random variables \cite{bubeck2013bounded}}.

\subsection{Proof of Theorem \ref{theorem:gapindependent}}
First, we state a concentration inequality that will be used in the proof. 

\begin{lemma} \label{app:concentration}  (Lemma 6 in \citet{abbasi2011improved}) 
	Consider an arm $a$ for which the rewards of objective $i$ are generated by a process $\{ R^i_{a}(t) \}_{t=1}^T$ with $\mu^i_{a}= \text{E} [R^i_{a}(t)]$, where the noise $R^i_{a}(t) - \mu^i_{a}$ is conditionally 1-sub-Gaussian. Let $N_{a}(T)$ denote the number of times $a$ is selected by the beginning of round $T$. 
	Let $\hat{\mu}_{a}(T) = \sum_{t=1}^{T-1} \text{I} (a(t) =a ) R^i_{a}(t) / N_a(T)$ for $N_a(T) >0$ and $\hat{\mu}_{a}(T) = 0$ for $N_a(T) = 0$.
	Then, for any $0 < \delta < AD$ with probability at least $1-\delta/ (AD)$ we have
	
	\begin{align*}
	\left| \hat{\mu}_{a}(T)  - \mu_a \right|  & \leq \sqrt{  \frac{1+N_a(T)}{N^2_a(T)} 
		\left( 1 + 2 \log \left(  \frac{ AD \sqrt{1 + N_a(T) } } {\delta}    \right)  
		\right)  }  ~~ \forall T \in \mathbb{N}.  
	\end{align*}
\end{lemma}

Let 
$\text{UC}^i_{a} := \cup_{t=1}^{T} \{ \mu^i_a \notin[ l^i_{a}(t)  , u^i_{a}(t) ] \}$, 
$\text{UC}^i := \cup_{a \in {\cal A}} \text{UC}^i_{a}$ 
and $\text{UC} := \cup_{i \in {\cal D}} \text{UC}^i$.
The following lemma bounds the probability of $\text{UC}$. 

\begin{lemma} \label{lmm:prUC}
	$\Pr ( \text{UC} ) \leq \delta $.
\end{lemma}
\begin{proof}
	This follows from Lemma \ref{app:concentration}. We observe that $\{ \mu^i_a \in[ l^i_{a}(t)  , u^i_{a}(t) ] \} =  \{ | \mu^i_a - \hat{\mu}^i_a(t) | \leq c_a(t)  \}$. Thus, Lemma \ref{app:concentration} shows that $\neg \text{UC}^i_a$ holds with probability at least $1 - \delta/(AD)$, and hence, $\text{UC}^i_a$ holds with probability at most $\delta/(AD)$. Applying the union bound, we get $\Pr ( \text{UC} ) \leq \delta $.
\end{proof}

Let ${\cal T} := \{ 1 \leq t \leq T : \forall a\in \hat{\mathcal{A}}_*^1(t): c_a(t)\leq\epsilon/2 \}$ denote the set of rounds in which PF-LEX selects the arm $\hat{a}_*^D(t)$ and $\neg{\cal T} := \{1,\ldots,T\} - {\cal T}$. In the following lemma, the gap of the arm selected in round $t \in {\cal T}$ in objective $i$ is bounded as a function of $\epsilon$ and the length of the confidence interval of the selected arm on event $\neg \text{UC}$ if the selected arm is in $\mathcal{S}_*^i$. 

\begin{lemma} \label{lmm:gap-bound}
	When PF-LEX is run, the following holds on event $\neg \text{UC}$ if $a(t)\in\mathcal{S}_*^i$: $\mu_*^i-\mu_{a(t)}^i \leq u_{a(t)}^i(t)-l_{a(t)}^i(t) + \epsilon(A-1)$ for $t \in {\cal T}$.
\end{lemma}
\begin{proof}
	Consider any lexicographic optimal arm $*$. We have 
	\begin{align}
		\mu_*^i-\mu_{a(t)}^i &\leq u_*^i(t)-l_{a(t)}^i(t) \label{eqn:lgb-1} \\
		&\leq u_{\hat{a}_*^i(t)}^i(t)-l_{a(t)}^i(t) \label{eqn:lgb-2} \\
		&\leq u_{a(t)}^i(t)-l_{a(t)}^i(t) + \epsilon(A-1) \label{eqn:lgb-3} ~.
	\end{align}
   Here \eqref{eqn:lgb-1} holds since $\mu_*^i\leq u_*^i(t)$ and $\mu_{a(t)}^i\geq l_{a(t)}^i(t)$ on event $\neg\text{UC}$. \eqref{eqn:lgb-2} holds by the definition of $\hat{a}_*^i(t)$ and the fact that $*\in\hat{\mathcal{A}}_*^{i-1}(t)$, which is proven by induction. For this, consider any objective $j\in\{1,\ldots,i-1\}$. We first observe that $*\in\hat{\mathcal{A}}_*^0(t)$ and $a(t)=\hat{a}_*^D(t)\in\hat{\mathcal{A}}_*^{D-1}(t)\subseteq\hat{\mathcal{A}}_*^{j}(t)$. Next, we show that $*\in\hat{\mathcal{A}}_*^{j-1}(t) \implies *\in\hat{\mathcal{A}}_*^{j}(t)$ to conclude that $*\in\hat{\mathcal{A}}_*^{i-1}(t)$. Since $a(t)\in\mathcal{S}_*^i$, $\mu_{a(t)}^j=\mu_*^j$, which implies that $a(t)$ and $*$ are linked in objective $j$. Since $a(t)\in\hat{\mathcal{A}}_*^{j}(t)$, $a(t)$ is chained to $\hat{a}_*^j(t)$ in objective $j$, which implies that $*$ is chained to $\hat{a}_*^j(t)$ in objective $j$ as well.
	Finally, if $i=D$, \eqref{eqn:lgb-3} holds trivially as $a(t)=\hat{a}_*^i(t)=\hat{a}_*^D(t)$. Otherwise, since $a(t)=\hat{a}_*^D(t)\in\hat{\mathcal{A}}_*^{D-1}(t)\subseteq\hat{\mathcal{A}}_*^{i}(t)$, $a(t)$ is chained to $\hat{a}_*^i(t)$, which implies that $|u_{\hat{a}_*^i(t)}^i(t)-u_{a(t)}^i(t)| \leq 2(|\hat{\mathcal{A}}_*^{i}(t)|-1)\max_{a\in\hat{\mathcal{A}}_*^{i}(t)}c_a(t) \leq \epsilon(A-1)$.
\end{proof}

We also need to bound the regret in objective $i$ for rounds up to round $T$ for which $t \notin {\cal T}$. Let $\neg{\cal T}_a := \{t \in \neg\mathcal{T} : a(t)=a \}$. Obviously, PF-LEX does not incur any regret in objective $i$ in rounds $t \in \neg{\cal T}_a$ for $a \in {\cal A} - {\cal S}^i_*$, and incurs regret $\Delta^i_a$ in objective $i$ in rounds $t \in \neg{\cal T}_a$ for $a \in {\cal S}^i_*$.

\begin{lemma} \label{lmm:round-bound}
	When PF-LEX is run, we have 
	\begin{align*}
		\sum_{t \in \neg{\cal T}} \In\{a(t)\in\mathcal{S}_*^i\} \Delta^i_{a(t)} &\leq 
		\sum_{ a \in {\cal S}^*_i } \left(3 + \frac{16}{\epsilon^2} \log \frac{2\sqrt{e} AD}{\epsilon \delta} \right)  \Delta^i_{a}
	\end{align*}
	for all objectives $i\in\mathcal{D}$.
\end{lemma} 
\begin{proof}
	The proof follows from bounding the cardinality of $\neg{\cal T}_a$ for $a \in {\cal S}^i_*$. Note that $t \in \neg{\cal T}_a$ happens only when $c_{a}(t) > \epsilon/2$. Similar to the proof of Theorem 7 in \citet{abbasi2011improved}, this implies that 
	\begin{align}
		\frac{N^2_a(t)-1}{N_a(t)+1} \leq \frac{N^2_{a}(t)}{N_a(t)+1} \leq 
		\frac{4}{\epsilon^2} \left(1+ 2 \log \frac{AD\sqrt{1+N_a(t)}}{\delta} \right) \nonumber ~.
	\end{align}
	Then, from Lemma 8 in \citet{antos2010}, we obtain $N_a(t) \leq 3 + \frac{16}{\epsilon^2} \log \frac{2\sqrt{e} AD}{\epsilon \delta} $.
\end{proof}

In the remaining part, we bound $\Reg_{pb}^i(T)$ under the event $\neg \text{UC}$ and $\Ex [\Reg_{pb}^i(T) ]$ by using the results of the lemmas above. For the latter, we observe that:
\begin{align}
	\Ex[ \Reg_{pb}^i(T) ] &= \Ex [\Reg_{pb}^i(T) | \text{UC} ] \Pr ( \text{UC} ) + \Ex [ \Reg_{pb}^i(T) | \neg \text{UC} ] \Pr ( \neg \text{UC} ) \nonumber \\
	&\leq T \Delta^i_{\max} \Pr ( \text{UC}  ) + \Ex [ \Reg_{pb}^i(T) | \neg \text{UC} ] \label{eqn:partitiondecompose2} ~.
\end{align}

For each $i\in\mathcal{D}$, the bound for $\Reg_{pb}^i(T)$ is obtained by using the result in Lemmas \ref{lmm:gap-bound} and \ref{lmm:round-bound}. By Lemma \ref{lmm:round-bound}, we know that 
\begin{align}
	\sum_{t \in \neg{\cal T}} \In\{a(t)\in\mathcal{S}_i^*\} \Delta^i_{a(t)} &\leq 3 |{\cal S}^i_*|  \Delta^i_{\max}  + \frac{16 |{\cal S}^i_*|  \Delta^i_{\max} }{\epsilon^2} \log \frac{2\sqrt{e} AD}{\epsilon \delta} ~. \label{eqn:gapindep2_1}
\end{align}
Let ${\cal N}_{a} := \{ t\in\mathcal{T} : a(t) = a \}$. By Lemma \ref{lmm:gap-bound}, on event $\neg\text{UC}$ (which happens with probability at least $1-\delta$), we have 
\begin{align}
\sum_{t \in {\cal T}} \In\{a(t)\in\mathcal{S}_i^*\} \Delta^i_{a(t)} &\leq  \sum_{a \in {\cal S}^i_*} \sum_{t \in {\cal N}_{a} } ( u^i_{a}(t) - l^i_{a}(t) ) + \epsilon(A-1)T \nonumber \\
&\leq 2 \sqrt{2} \sum_{a \in {\cal S}^i_*} \left( B_{T,\delta} \sum_{ t \in {\cal N}_{a}  }  \sqrt{ \frac{1}{N_{a}(t) } } \right) + \epsilon(A-1)T \notag \\
&\leq 2 \sqrt{2} B_{T,\delta} \sum_{a \in {\cal S}^i_*} \sqrt{N_a(T)} + \epsilon(A-1)T \notag  \\
& \leq 4 \sqrt{2} B_{T,\delta} \sqrt{  |{\cal S}^i_*|  T } +  \epsilon(A-1)T . \label{eqn:gapindep2_4}
\end{align}

The bound for $\Reg_{pb}^i(T)$ is obtained by summing the results of \eqref{eqn:gapindep2_1} and \eqref{eqn:gapindep2_4}. Finally, the bounds on the expected regret simply follows from using \eqref{eqn:partitiondecompose2} and setting $\delta = 1/T$.

\subsection{Tables of Notation}
General notation is listed in Table \ref{tbl:notation-system}. Notations specific to each case covered in Section \ref{sec:case1} are listed in Tables \ref{tbl:notation-case1}, \ref{tbl:notation-case2} and \ref{tbl:notation-case3} respectively.

\begin{table}[!ht]
    \caption{List of system notations.}
    \label{tbl:notation-system}
    \smallskip
    \centering
    \begin{tabular}{lp{2in}p{2.4in}}
        \toprule
        \textbf{Notation} & \textbf{Definition} & \textbf{Description} \\
        \midrule
        $A$ & & Number of arms \\
        $\mathcal{A}$ & $[A]$ & Set of all arms \\
        $a(t)$ & & Selected arm in round $t$ \\
        $D$ & & Number of objectives \\
        $\mathcal{D}$ & $[D]$ & Set of all objectives \\
        $\mu_a^i$ & & Expected reward of arm $a$ in objective $i$ \\
        $\bm{\mu}_a$ & $(\mu_a^1,...,\mu_a^D)$ & Expected reward vector of arm $a$ \\
        $\kappa^i(t)$ & & Noise in objective $i$ in round $t$ \\
        $r^i(t)$ & $\mu_{a(t)}^i+\kappa^i(t)$ & Reward of the selected arm in round $t$ \\
        $\succ_{\lex,i}$ & $\bm{\mu} \succ_{\lex,i} \bm{\mu'} \iff \mu^j > \mu'^j$, $j = \min\{k\leq i: \mu^k\neq\mu'^k\}$ & Symbol for lexicographic dominance in the first $i$ objectives \\
        $\mathcal{A}_*^i$ & $\{a:\mathcal{A}: \bm{\mu}_{a'}\nsucc_{\lex,i}\bm{\mu}_{a}, \forall a'\in\mathcal{A} \}$ & Set of lexicographic optimal arms in the first $i$ objectives \\
        $\mathcal{A}_*$ & $\mathcal{A}_*^D$ & Set of lexicographic optimal arms in all objectives \\
        $*$ & $*:*\in\mathcal{A}_*$ & A lexicographic optimal arm \\
        $\Delta_a^i$ & $\mu_*^i-\mu_a^i$ & \rev{Gap} of arm $a$ in objective $i$ \\
        $\nabla_a^i$ & $|\mu_*^i-\mu_a^i|$ & \rev{Absolute gap} of arm $a$ in objective $i$ \\
        $\mathcal{S}_*^i$ & $\mathcal{A}_*^{i-1}-\mathcal{A}_*^i$ & Set of arms that are lexicographic optimal in the first $i-1$ objectives but not lexicographic optimal in the first $i$ objectives \\
        \rev{${\cal S}^i$} & $\{ a : \Delta^i_a > 0 \}$ & Set of suboptimal arms in objective $i$. \\
        $\Delta_{\max}^i$ & $\max_{a\in\mathcal{A}} \Delta_a^i$ & Maximum suboptimality gap in objective $i$ \\
        $\nabla_a^{\max}$ & $\max_{i\in\mathcal{D}} \nabla_a^i$ & Maximum absolute gap of arm $a$ \\
        $\bm{\Reg}_{pb}(T)$ & $(\Reg^1_{pb}(T),...,\Reg^D_{pb}(T))$ & Lexicographic priority-based regret \\
        $\bm{\Reg}_{pf}(T)$ & $(\Reg^1_{pf}(T),...,\Reg^D_{pf}(T))$ & Lexicographic priority-free regret \\    
        $\Reg^i_{pb}(T)$ & $\sum_{t=1}^T \Delta_{a(t)}^i \In\{a(t)\in\mathcal{S}_*^i\} $ & Priority-based regret in objective $i$ \\
        $\Reg^i_{pf}(T)$ & $\sum_{t=1}^T \Delta_{a(t)}^i$ & Priority-free regret in objective $i$ \\  
        $\Delta_{\min}^i$ & $\min_{a\in\mathcal{S}_*^i} \Delta_a^i$ for $\bm{\Reg}_{pb}(T)$, $\min_{a\in\mathcal{S}^i} \Delta_a^i$ for $\bm{\Reg}_{pf}(T)$ & Minimum suboptimality gap in objective $i$ \\
        \bottomrule
    \end{tabular}
\end{table}

\begin{table}[!ht]
    \caption{List of notations for Case 1.}
    \label{tbl:notation-case1}
    \smallskip
    \centering
    \begin{tabular}{lp{2in}p{2.4in}}
        \toprule
        \textbf{Notation} & \textbf{Definition} & \textbf{Description} \\
        \midrule
        $N_a(t)$ & $\sum_{t'=1}^{t-1}\In\{a(t')=a\}$ & Number of times arm $a$ was selected by the beginning of round $t$ \\
        $\hat{\mu}_a^i(t)$ & $\sum_{t'=1}^{t-1}r^i(t')/N_a(t)$ & Sample mean of the rewards of arm $a$ in objective $i$ at the beginning of round $t$ \\
        $\hat{\mathcal{A}}_*(t)$ & $\{a\in\mathcal{A}:\forall i\in\mathcal{D}, |\hat{\mu}_a^i(t)|<\sqrt{4\log N_a(t)/N_a(t)}\}$ & Set of estimated lexicographic optimal arms at the beginning of round $t$ \\
        $\dagger(a)$ & $\argmax_{i\in\mathcal{D}} \nabla_a^i$ & Objective for which $\nabla_a^{\dagger(a)}=\nabla_a^{\max}$ \\
        $(\cdot)_a^{\dagger}$ & $(\cdot)_a^{\dagger(a)}$ & \\
        \bottomrule
    \end{tabular}
\end{table}

\begin{table}[!ht]
    \caption{List of notations for Case 2.}
    \label{tbl:notation-case2}
    \smallskip
    \centering
    \begin{tabular}{lp{2in}p{2.4in}}
        \toprule
        \textbf{Notation} & \textbf{Definition} & \textbf{Description} \\
        \midrule
        $\eta_i$ & $\eta_i: \mu_*^i-\Delta_{\min}^i<\eta_i <\mu_*^i$ & Near-lexicographic optimal expected reward in objective $i$ \\
        $\delta_i$ & $\mu_*^i-\eta_i$ & Gap between near-lexicographic optimal expected reward and the lexicographic optimal expected reward in objective $i$ \\
        $N_a(t)$ & $\sum_{t'=1}^{t-1}\In\{a(t')=a\}$ & Number of times arm $a$ was selected by the beginning of round $t$ \\
		$\hat{\mu}_a^i(t)$ & $\sum_{t'=1}^{t-1}r^i(t')/N_a(t)$ & Sample mean of the rewards of arm $a$ in objective $i$ at the beginning of round $t$ \\
        $\hat{\mathcal{A}}_*(t)$ & $\{a\in\mathcal{A}:\forall i\in\mathcal{D}, \hat{\mu}_a^i(t) > -\sqrt{4\log N_a(t)/N_a(t)} \}$ & Set of estimated lexicographic optimal arms at the beginning of round $t$ \\
        $\dagger(a)$ & $\argmax_{i\in\mathcal{D}} (\Delta_a^i-\delta_i)$ & \\
        $(\cdot)_a^{\dagger}$ & $(\cdot)_a^{\dagger(a)}$ & \\
        \bottomrule
    \end{tabular}
\end{table}

\begin{table}[!ht]
    \caption{List of notations for Case 3.}
    \label{tbl:notation-case3}
    \smallskip
    \centering
    \begin{tabular}{lp{2in}p{2.4in}}
        \toprule
        \textbf{Notation} & \textbf{Definition} & \textbf{Description} \\
        \midrule
        $\epsilon$ & & Suboptimality that is aimed to be tolerated \\
        $\delta$ & & Confidence term \\ 
        $N_a(t)$ & $\sum_{t'=1}^{t-1}\In\{a(t')=a\}$ & Number of times arm $a$ was selected by the beginning of round $t$ \\
		$\hat{\mu}_a^i(t)$ & $\sum_{t'=1}^{t-1}r^i(t')/N_a(t)$ & Sample mean of the rewards of arm $a$ in objective $i$ at the beginning of round $t$ \\
        $c_a(t)$ & See Section \ref{sec:case1} & Half of the length of the confidence interval of arm $a$ at the beginning of round $t$ \\
        $u_a^i(t)$ & $\hat{\mu}_a^i(t)+c_a(t)$ & Upper confidence bound of arm $a$ in objective $i$ at the beginning of round $t$ \\
        $l_a^i(t)$ & $\hat{\mu}_a^i(t)-c_a(t)$ & Lower confidence bound of arm $a$ in objective $i$ at the beginning of round $t$ \\
        $C_{i,t}$ & See Section \ref{sec:case1} & Chains to in objective $i$ in round $t$ \\
        $\hat{a}_*^i(t)$ & $\argmax_{a\in\hat{\mathcal{A}}_*^{i-1}(t)} u_a^i(t)$ & Optimistic near-lexicographic optimal arm in objective $i$ at the beginning of round $t$ \\
        $\hat{\mathcal{A}}^i_*(t)$ & $\{a\in\mathcal{A}:a\:C_{i,t}\:\hat{a}_*^i(t)\}$ & Set of estimated lexicographic optimal arms in the first $i$ objectives at the beginning of round $t$ \\
        \bottomrule
    \end{tabular}
\end{table}

\end{document}